%% file: main.tex
\documentclass[runningheads]{llncs}

\usepackage{eccv}

\usepackage{eccvabbrv}

\usepackage{amssymb}
\usepackage{amsmath}

\usepackage{amsthm}
\usepackage{graphicx}

\usepackage{booktabs}
\usepackage{mathtools, nccmath}
\usepackage[accsupp]{axessibility}  

\usepackage{multirow}
\usepackage{pifont}

\usepackage[utf8]{inputenc}
\usepackage[T1]{fontenc}
\usepackage{url}
\usepackage{booktabs}
\usepackage{amsfonts}
\usepackage{nicefrac}
\usepackage{microtype}

\usepackage{braket} 
\usepackage{color}

\usepackage{float}
\usepackage{adjustbox}
\usepackage{caption} 
\usepackage{mathtools, nccmath}

\usepackage{wrapfig}
\usepackage{tcolorbox}
\usepackage{thmtools}
\usepackage{thm-restate}
\usepackage{color,soul}

\input{math_commands.tex}

\input{macros}

\renewcommand{\vec}[1]{\overset{\rightarrow}{#1}}

\usepackage[pagebackref,breaklinks,colorlinks,citecolor=eccvblue]{hyperref}

\begin{document}

\title{Equivariant Local Reference Frames for Unsupervised Non-rigid Point Cloud Shape Correspondence} 

\titlerunning{\textsc{EquiShape}}

\author{Ling Wang\inst{*1,2}
\and
Runfa Chen\inst{*2}
\and
Yikai Wang$^{\dagger}$\inst{2}
\and
Fuchun Sun$^{\dagger}$\inst{2}
\and
Xinzhou Wang\inst{2,3}
\and
Sun Kai\inst{2}
\and
Guangyuan Fu\inst{1}
\and
Jianwei Zhang\inst{4}
\and
Wenbing Huang\inst{5}}
\authorrunning{Wang et al.}

\institute{Xi’an Research Institute of High-Tech \and
Tsinghua University \and
Tongji University \and
University of Hamburg
\and
Renmin University of China \\
\email{yanyuwangl@gmail.com}, \\
\email{\{crf21,wangyk17,fcsun\}@mails.tsinghua.edu.cn}}
\maketitle
\begingroup
\renewcommand\thefootnote{}\footnote{\text{*} Indicates Equal Contribution. \text{\textdagger} Indicates Corresponding Author.
}
\addtocounter{footnote}{1}
\endgroup

\input{sec/0-abstract}
\input{sec/1-intro}

\input{sec/2-related}

\input{sec/3-preli}
\input{sec/4-method}

\input{sec/5-experiment}
\input{sec/6-conclusion}
\bibliographystyle{splncs04}
\bibliography{main}
\input{sec/X_suppl}
\end{document}

%% file: math_commands.tex
\usepackage{amsmath,amsfonts,bm}

\def\eqref#1{equation~\ref{#1}}

\def\1{\bm{1}}

\def\va{{\bm{a}}}
\def\vb{{\bm{b}}}
\def\vc{{\bm{c}}}

\def\ve{{\bm{e}}}

\def\vh{{\bm{h}}}

\def\vp{{\bm{p}}}

\def\vt{{\bm{t}}}
\def\vu{{\bm{u}}}
\def\vv{{\bm{v}}}

\def\vx{{\bm{x}}}
\def\vy{{\bm{y}}}

\def\mI{{\bm{I}}}

\def\mO{{\bm{O}}}

\def\mX{{\bm{X}}}
\def\mY{{\bm{Y}}}
\def\mZ{{\bm{Z}}}

\DeclareMathAlphabet{\mathsfit}{\encodingdefault}{\sfdefault}{m}{sl}
\SetMathAlphabet{\mathsfit}{bold}{\encodingdefault}{\sfdefault}{bx}{n}

\def\gD{{\mathcal{D}}}
\def\gE{{\mathcal{E}}}

\def\gV{{\mathcal{V}}}

\def\sR{{\mathbb{R}}}



%% file: macros.tex
\newcommand{\E}{\mathbb{E}}

\newcommand{\R}{\mathbb{R}}

\renewcommand{\epsilon}{\ensuremath\varepsilon}

\renewcommand{\phi}{\ensuremath{\varphi}}

 \newcommand{\hb}{\mathbf{h}}

 \newcommand{\Wb}{\mathbf{W}}

\newcommand{\Ecal}{\mathcal{E}}

\newcommand{\Gcal}{\mathcal{G}}

\newcommand{\Vcal}{\mathcal{V}}

\newcommand*{\argmin}{\mathop{\mathrm{argmin}}}



%% file: sec/0-abstract.tex
\vspace{-1.7em}
\begin{abstract}
Unsupervised non-rigid point cloud shape correspondence underpins a multitude of 3D vision tasks, yet itself is non-trivial given the exponential complexity stemming from inter-point degree-of-freedom, \ie, pose transformations.
Based on the assumption of local rigidity, one solution for reducing complexity is to decompose the overall shape into independent local regions using Local Reference Frames (LRFs) that are invariant to SE(3) transformations. 
However, the focus solely on local structure neglects global geometric contexts, resulting in less distinctive LRFs that lack crucial semantic information necessary for effective matching. 
Furthermore, such complexity introduces out-of-distribution geometric contexts during inference, thus complicating generalization.
To this end,  we introduce 1) \textsc{EquiShape}, a novel structure tailored to learn pair-wise LRFs with global structural cues for both spatial and semantic consistency, and 2) LRF-Refine, an optimization strategy generally applicable to LRF-based methods, aimed at addressing the generalization challenges. 
Specifically, for \textsc{EquiShape}, we employ cross-talk within separate equivariant graph neural networks (Cross-GVP) to build long-range dependencies to compensate for the lack of semantic information in local structure modeling, deducing pair-wise independent SE(3)-equivariant LRF vectors for each point. 
For LRF-Refine, the optimization adjusts LRFs within specific contexts and knowledge, enhancing the geometric and semantic generalizability of point features. Our overall framework surpasses the state-of-the-art methods by a large margin on three benchmarks. Code and models will be publicly available.
\vspace{-.5em}
\keywords{Non-rigid Shape Correspondence \and Local Reference Frames  \and Generalization \and Equivariant Neural Networks}
\end{abstract}

\begin{figure}[!t]
  \begin{center}
      \includegraphics[width=1\textwidth]{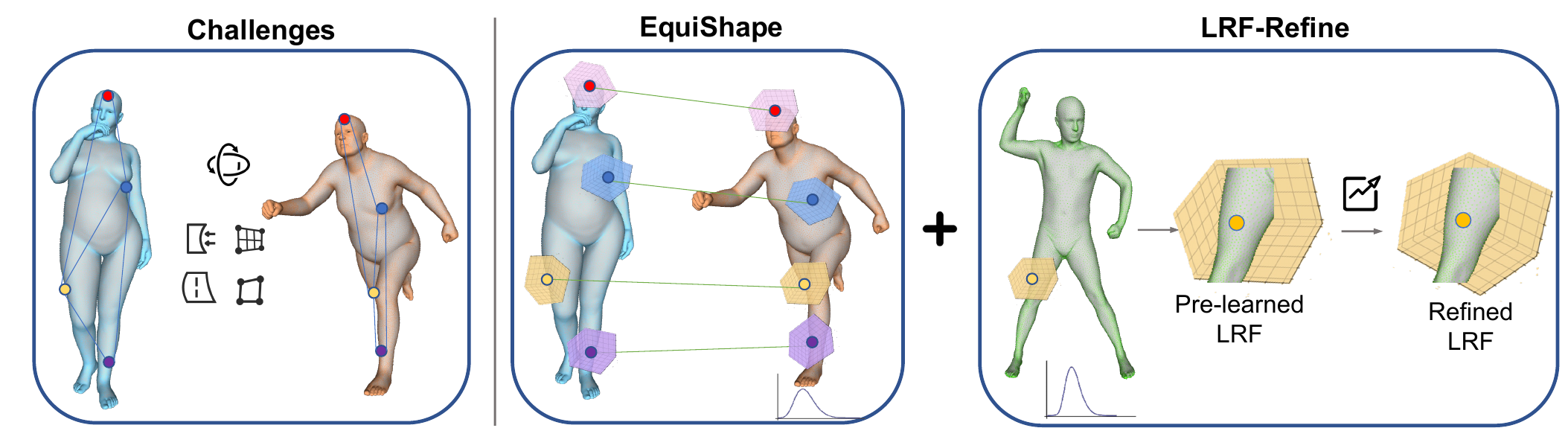}
      \vspace{-1.5em}
      \caption{ Illustration of our insight, which includes: 1) The challenge of exponential variations in shape deformation and orientation caused by inter-point pose transformations; 2) EquiShape addresses such challenge by decomposing the overall shape into independent local regions using LRFs invariant to SE(3) transformations and builds long-range dependencies to compensate for the lack of semantic information in local structure modeling; 3) LRF-Refine optimizes LRFs for adaptation to out-of-distribution geometric contexts, which inevitably arise from such challenge during inference.}

      
      \vspace{-2.5em}
      \label{fig:teaser}
  \end{center}
\end{figure}

%% file: sec/1-intro.tex
\vspace{-2.2em}
\section{Introduction}
\label{sec:intro}
\vspace{-.5em}

Non-rigid shape matching, which predicts point-to-point correspondence between deformable 3D objects, is pivotal for a wide range of 3D tasks, including 3D deformation transfer \cite{zhang2023self,deng2021deformed}, motion tracking\cite{bozic2020neural}, 4D reconstruction \cite{wang2023root,lei2022cadex}, and shape editing\cite{sumner2007embedded}, \etc. 
However, the reliance on dense annotations presents a significant challenge, particularly when it is difficult to obtain these dense labels in practice. Consequently, recent trends have shifted away from methods dependent on pre-annotated guides \cite{donati2020deep,loper2023smpl} towards unsupervised settings \cite{groueix20183d,zeng2021corrnet3d,lang2021dpc}, which utilize deep neural networks to infer shape correspondences directly from raw, unannotated data. 
A paramount challenge in non-rigid shape matching arises from the exponential variations in shape deformation and orientation~\cite{chi2021garmentnets}, caused by inter-point pose transformations within and between shapes.
Existing methods commonly assume that the orientation of all shapes is pre-aligned to a canonical orientation to simplify handling variations in orientation~\cite{eisenberger2020smooth,cao2023self, sharma2020weakly}.
However, these assumptions frequently fall short in practical unsupervised scenarios~\cite{donati2022deep}, which limits the ability to generalize across diverse datasets.

To alleviate this, some existing methods
~\cite{deng2023se,jiang2023non} utilize data augmentation and approximation techniques to increase the diversity of training data. Yet, these approaches cannot essentially solve it since they may not be able to capture all possible orientation variations, especially for those that are complex or unique to a particular dataset.
Furthermore, the challenge of shape deformation complexity remains underexplored. 
One promising strategy to address these complexity issues is predicated on the assumption of local rigidity~\cite{qin2023deep,zhang2023self}, to disentangle the global pose transformation into independent local views for each local region, remarkably reducing the global search space. 
Previous practice in rigid-shape analysis learns rotation-invariant/equivariant descriptors for each point, \eg, PerfectMatch~\cite{gojcic2019perfect}, typically relying on hand-crafted Local Reference Frames (LRFs) which is usually based on the covariance analysis of the local surface, to transform local patches to a defined canonical representation.
This comprehensive strategy leads to enhanced accuracy across rigid datasets.
However, hand-crafted LRFs~\cite{tombari2010unique}  are limited in some scenarios with complicated geometry, which are observed to be prone to poor performance in our experiments.

The advent of EquivariantMP \cite{luo2022equivariant} marks a significant shift by integrating SE(3)-equivariant networks into LRF modeling in an end-to-end learning paradigm for point cloud analysis, significantly diminishing the dependency on hand-crafted design decisions.
However, to the best of our knowledge, no existing work has extended EquivariantMP to non-rigid shape matching. The reasons behind this might be twofold: 
1) Learning equivariant LRFs for pairwise or multiple shapes is non-trivial since it should consider both the local structure and all possible global geometric contexts, while maintaining the independent equivariance property between shapes. Here, the global geometric contexts refer to the global information from absolute coordinates and semantic consistency with the counterpart shape. 2) 
Spatial and semantic consistencies in the feature space are essential for effective matching, which necessitate LRFs addressing Out-Of-Distribution (OOD) generalization challenges. Such LRFs ensure corresponding local regions learn consistent semantic features within an aligned coordinate space. Here, the OOD challenges arise from the nearly infinite distribution space, leading the model to encounter geometric contexts during inference that are markedly different from those seen during training.

In response, we propose a novel framework \textbf{\textsc{EquiShape}} designed for non-rigid point cloud shape correspondence.
At its core, \textsc{EquiShape} employs the cross-talk within separate Geometric Vector Perceptrons (Cross-GVP) to learn pair-wise independent SE(3)-equivariant LRFs, where the global information from absolute coordinates is integrated into equivariant vector features. 
Besides, the cross-attention is utilized to aggregate invariant scale features from both shapes to compose the novel globally-aware equivariant and invariant features, while keeping the independent equivariance property between shapes. This approach broadens the range of the contextual clues for LRF learning, compensating for the semantic information gap in local structure modeling.

Additionally, adapting LRFs to specific contextual inputs is essential for further enhancing the feature-matching capabilities. We introduce an LRF Refinement Optimization (\textbf{LRF-Refine}) that applies gradient descent through the frozen model to directly optimize the residual LRF vectors, enabling LRF adaptation to specific contexts under the constraints of model knowledge. This refinement strategy enhances pre-learned features, allowing them to successfully generalize to specific scenes and consistently and significantly outperform existing approaches across a wide array of datasets.

To summarize, the contributions of our work include:
\vspace{-1.em}
\begin{itemize}
\item
To address the issue of exponential complexity arising from inter-point pose transformations, we propose \textsc{EquiShape}, equipped with a specially designed Cross-GVP. This model learns pairwise independent SE(3)-equivariant LRFs for each point, enabling the descriptors to potentially be decoupled from inter-point pose transformations while integrating sufficient global geometric contexts. To the best of our knowledge, this is the first work to incorporate equivariant networks in the field of non-rigid shape matching, marking a novel approach in this domain.

\item 
In response to the inherent challenges posed by OOD geometric contexts, exacerbated by extensive shape variations, we have incorporated LRF-Refine, an optimization strategy applicable to LRF-based methods. This refinement process adjusts the LRF vectors to specific inputs,  under the guidance of model constraints, thereby substantially improving the geometric and semantic generalizability of point features.

\item Substantial comparison results verify the efficacy of \textsc{EquiShape}, demonstrating a significant advancement in accuracy, robustness, and generalization over existing methods. 
\end{itemize}

%% file: sec/2-related.tex
\section{Related Work}
\label{sec:related}

\paragraph{LRF Method}
The efficacy of 3D rotation-invariant descriptors, like SHOT~\cite{tombari2010unique}, RoPS~\cite{guo2013rops}, GFrames~\cite{melzi2019gframes}, and TOLDI~\cite{yang2017toldi}, hinges on a robust and repeatable LRF. Traditionally, LRFs are estimated through eigenvalue decomposition of the covariance matrix from neighborhood points or by using geometric features such as normals. Despite their inherent rotation invariance, these methods falter in complex and noisy environments~\cite{yu2024riga}. With the advent of deep learning, there's a shift towards data-driven descriptor design. However, reliance on handcrafted features persists, limiting the exploitation of end-to-end learning~\cite{gojcic2019perfect,spezialetti2019learning}. Our \textsc{EquiShape} leverages the equivariant graph neural network, ensuring SE(3) equivariance and distinctiveness through innovative end-to-end LRF learning and long-range information integration.


\paragraph{Unsupervised Non-Rigid Shape Matching}
Unsupervised non-rigid shape matching methods are broadly categorized into spectral-based and spatial-based approaches. The former, predominantly grounded in the functional maps paradigm, has seen recent advancements with deep functional map methods \cite{spezialetti2019learning, cao2022unsupervised, cao2023self, jiang2023non} yielding state-of-the-art results for mesh-represented shapes. However, these spectral-based methods rely on topology information, which is less accessible in real-world captures and generally constrained to near-isometric deformation shape matching.
On the other hand, spatial-based methods require only the coordinates of point clouds ~\cite{halimi2019unsupervised, deprelle2019learning,zeng2021corrnet3d,lang2021dpc,deng2023se}, steering our focus towards point cloud shape correspondence. CorrNet3D~\cite{zeng2021corrnet3d}, for instance, explores an unsupervised encoder-decoder architecture for learning point cloud shape correspondence. DPC \cite{lang2021dpc} introduces self- and cross- losses for smooth correspondence mapping. Yet, previous methods have not adequately addressed discrepancies in orientations.
SE-ORNet \cite{deng2023se} tries to address this via data augmentation, essentially approximating rather than truly embodying equivariance. In contrast, our \textsc{EquiShape} innovatively incorporates geometric symmetry into the network design, ensuring that our model is theoretically precise in achieving SE(3)-equivariance.

%% file: sec/3-preli.tex
\section{Preliminaries}
\label{sec:preli}


\begin{definition}[\(\text{SE}(3)\)-equivariance]
    \label{def:equ}
    Let $\Vec{\mZ}$ denote 3D geometric vectors (such as positions, etc.) that are steerable by \(\text{SE}(3)\) transformations, and let $\vh$ represent non-steerable features (like distances, etc.).
    A function $f$ is \(\text{SE}(3)\)-equivariant if, for any transformation $g \in \text{SE}(3)$, it holds that $f(g \cdot \Vec{\mZ}, \vh) = g \cdot f(\Vec{\mZ}, \vh)$, $\forall \Vec{\mZ} \in \mathbb{R}^{3 \times m}$ and $\vh \in \mathbb{R}^{d}$. Similarly, $f$ is \(\text{SE}(3)\)-invariant if $f(g \cdot \Vec{\mZ}, \vh) = f(\Vec{\mZ}, \vh)$. 
\end{definition}
    
In Definition~\ref{def:equ}, the group action $\cdot$ is instantiated as $g\cdot\Vec{\mZ}\coloneqq \mO\Vec{\mZ}$ for the special orthogonal group that only consists of rotations where $\mO\in \text{SO}(3)\coloneqq\{\mO\in\sR^{3\times 3}|\mO^\top\mO=\mI, \det \mO = 1\}$, and is additionally implemented as the translation $g\cdot \vec{\vx} \coloneqq \vec{\vx} + \vec{\vt}$ for the 3D coordinate vector where $\vt\in\R^{3}$. Note that for the input of $f$, we have added the right-arrow superscript on $\Vec{\mZ}$ to distinguish it from the scalar $\vh$ that is unaffected by the transformation.



\begin{definition}[LRF]
Given a point cloud $\vec{\mathcal{P}}$, the LRF $\mO(\vec{\vp})$ at point $\vec{\vp} \in \vec{\mathcal{P}}$ is defined as
\begin{equation}
   \mO(\vec{\vp})=\{\vec{\mathbf{x}}(\vec{\vp}), \vec{\mathbf{y}}(\vec{\vp}), \vec{\mathbf{z}}(\vec{\vp})\}, 
\end{equation}
where $\vec{\mathbf{x}}(\vec{\vp}), \vec{\mathbf{y}}(\vec{\vp}), \vec{\mathbf{z}}(\vec{\vp})$ are the orthogonal axes of the coordinate system, satisfying the right-hand rule $\vec{\mathbf{y}}=\vec{\mathbf{z}} \times \vec{\mathbf{x}}$. Defining a LRF means defining a way to compute each one of its axes. 
\end{definition}
The desirable properties for LRF are twofold: Firstly, invariance to \(\text{SE}(3)\) transformations, such as translations and rotations, is crucial. Secondly, distinctiveness is essential for effective correspondence matching.

\paragraph{Unsupervised Non-Rigid Point Cloud Shape Correspondence}  We are given as input a pair of non-rigid shapes, which are denoted as the source and target, consisting of $n$ points, respectively. These shapes are represented as unordered 3D point clouds ${\Vec{\mX}}\in\R^{3\times n}, {\Vec{\mY}}\in\R^{3\times n}$, respectively. 
The goal of unsupervised non-rigid point cloud shape correspondence is to find the mapping from source to target, \ie, $f: {\Vec{\mX}} \rightarrow {\Vec{\mY}}$, without ground-truth correspondence annotations. 
We have a dataset $\gD=\{\Vec{\mX}_i, \Vec{\mY}_i\}_{i=1}^{N}$ that consists of $N$ unlabeled pairs of shapes. 

We can represent a pair of non-rigid shapes as a pair of graphs $\Gcal_{\Vec{\mX}}=(\Vcal_{\Vec{\mX}},\Ecal_{\Vec{\mX}})$ and $\Gcal_{\Vec{\mY}}=(\Vcal_{\Vec{\mY}},\Ecal_{\Vec{\mY}})$, similar to \cite{lang2021dpc,deng2023se}. Each node $i \in \Vcal_{\Vec{\mX}}$ or $j \in \Vcal_{\Vec{\mY}}$ represents one point and has 3D location coordinates $\vec{\vx}_i \in {\Vec{\mX}}$, $\vec{\vy}_j \in {\Vec{\mY}}$, respectively. The graph topology $\Ecal_{\Vec{\mX}}, \Ecal_{\Vec{\mY}}$ are given by a k-nearest-neighbor ($k$NN) graph using Euclidean distance of the original 3D point clouds ${\Vec{\mX}}, {\Vec{\mY}}$, respectively.

A graph neural networks (GNNs) $\varphi_{\theta}$ is then employed to learn high-dimensional point representations for correspondence $F_{\Vec{\mX}} \in \mathbb{R}^{c \times n}, F_{\Vec{\mY}} \in \mathbb{R}^{c \times n}$, where $c$ is the feature dimension, given the 3D point clouds ${\Vec{\mX}}, {\Vec{\mY}}$ and the graph topology $\Ecal_{\Vec{\mX}}, \Ecal_{\Vec{\mY}}$ as the source and target input\footnote{For simplicity, we omit the graph topology $\Ecal_{\Vec{\mX}}, \Ecal_{\Vec{\mY}}$ in the input, since the graph topology $\Ecal_{\Vec{\mX}}, \Ecal_{\Vec{\mY}}$ can be given by a $k$NN graph using Euclidean distance of the input ${\Vec{\mX}}, {\Vec{\mY}}$, respectively. }, \ie,
 \begin{align}
 \label{eq:graph}
     F_{\Vec{\mX}}, F_{\Vec{\mY}} = \varphi_{\theta}\left({\Vec{\mX}}, {\Vec{\mY}} \right).
 \end{align}

Then, we can get a similarity matrix $S(F_{\Vec{\mX}}, F_{\Vec{\mY}})
\in \mathbb{R}^{n \times n}$  to measure their correspondence (commonly by measuring proximity as the cosine of the angle between their representations). 
During inference, we set the closest point $\vec{\vy}_j^{*} \in {\Vec{\mY}}$ in the high-dimensional point representations space for each point $\vec{\vx}_i \in {\Vec{\mX}}$ as its corresponding point.
This mapping can be formulated as:
\begin{equation}
        \label{inference}
        f\left(\vec{\vx}_{i}\right)=\vec{\vy}_{j^{*}}, j^{*}=\underset{j}{\operatorname{argmax}} (s_{i j}),
\end{equation}
where $s_{i j}$ are the i$^{\text{th}}$ rows and j$^{\text{th}}$ columns of $S(F_{\Vec{\mX}}, F_{\Vec{\mY}})$.





\paragraph{Invariance}
In this work, we bypass the focus on permutation equivariance and invariance, commonly inherent in graph neural networks across point-wise and global features. However, to mitigate the challenge of learning an effective correspondence network within the extensive search space, characterized by data distribution variability in terms of coordinate frame systems,
we aim for the predicted correspondence, specifically the similarity matrix $S(\cdot, \cdot)$, to be independent of the initial positions and orientations of the two shapes.
Formally,
\begin{equation}
\begin{split}
    S\left(F_{g_1\cdot{\Vec{\mX}}}, F_{g_2\cdot{\Vec{\mY}}}\right) \equiv  S\left(F_{\Vec{\mX}}, F_{\Vec{\mY}}\right),
\end{split}
\label{eq:SE(3)-invariance}
\end{equation}
where, $\forall g_1, g_2 \in \text{SE}(3), \forall {\Vec{\mX}}, {\Vec{\mY}}\in\R^{3\times n}$.

%% file: sec/4-method.tex
\section{Method}
\label{sec:Method}



In this section, we present our novel framework, consisting of \textsc{EquiShape} and LRF-Refine. We first elaborate our entire \textsc{EquiShape} model in \Cref{subsec:model}, consisting of Cross-talk within separate Geometric Vector Perceptrons (Cross-GVP), Parameterized Local Reference Frame Transform (LRF-Transform), and Training objectives. In addition, we introduce Local Reference Frames Refinement Optimization (LRF-Refine) in \Cref{subsec:transformation_refinement}.

\subsection{\textsc{EquiShape Model}}
\label{subsec:model}

\paragraph{Overview} 
Given the input ${\Vec{\mX}}$, ${\Vec{\mY}}$, we are in demand of a highly expressive $\varphi_\theta$ to learn high-dimensional point representations for correspondence.
In accordance with \Cref{eq:SE(3)-invariance}, a viable strategy is to decouple the global space into independent local views that remain invariant to transformations, \eg, translations and rotations.
Specifically, we adopt the pair-wise LRF for each point, decomposing the single overall shape into local regions through \(k\)-nearest neighbor (\(k\)NN) graph, with each region being independently transformed by LRF \(\mO\), derived from Cross-GVP. This model facilitates local matching within a common space in a pair-wise independent \(\text{SE}(3)\)-equivariant manner, 
considering \textbf{correlated} semantic consistency, \eg, global information from absolute coordinates and long-range dependencies with the counterpart shape, while maintaining \textbf{independent} spatial consistency of the two shapes. This is a distinct departure from frameworks like EquivariantMP~\cite{luo2022equivariant} which focus solely on general point cloud analysis without considering these critical aspects.  The following paragraphs detail each component of this model.

\paragraph{Cross-talk within separate Geometric Vector Perceptrons (Cross-GVP)}
In particular, our Cross-GVP processes the following operations in each computation.

We first separately construct $k$NN non-rigid shape static graphs $\Gcal_{\Vec{\mX}}=(\Vcal_{\Vec{\mX}},\Ecal_{\Vec{\mX}})$ and $\Gcal_{\Vec{\mY}}=(\Vcal_{\Vec{\mY}},\Ecal_{\Vec{\mY}})$, which is aligned with methodologies previously established in the literature~\cite{ginzburg2021dwc, lang2021dpc}.

Then, we define the edge properties and the processed node input for source shapes as follows:
\begin{equation}
\begin{aligned}
    \label{eq:edge}
     \vh_{ij} = \|\vec{\vx}_i - \vec{\vx}_j\|_2, 
     \vec{\mZ}_{ij} =  \vec{\vx}_i \ominus \vec{\vx}_j, \quad (i,j) \in \gE_{\Vec{\mX}},\\
\end{aligned}
\end{equation}
\begin{equation}
\begin{aligned}
    \label{eq:init}
    \vh^{(0)}_i = 0, 
    \vec{\mZ}^{(0)}_i = \vec{\vx}_i \ominus \vec{\vc_x}, \quad i \in \Vcal_{\Vec{\mX}},
\end{aligned}
\end{equation}
where the operation ``$\ominus$'' transforms the input positions into translation invariant representations by subtraction. $\vec{\vc_x} = \frac{1}{|\Vcal_{\Vec{\mX}}|} \sum_{i \in \Vcal_{\Vec{\mX}}} \vec{\vx}_i$ is the mean position of the points within the source point cloud. The superscript (0) indicates the initial state of the node inputs. 
Notably, EquivariantMP~\cite{luo2022equivariant} focuses solely on the local structure, such as edge distances and direction vectors, while neglecting global coordinates $\vec{Z}^{(0)}_i$. However, in non-rigid shape matching, these coordinates represent a fundamental global geometric context that contains rich semantic information.

\begin{figure*}[t!]
\centering
\includegraphics[width=\linewidth]{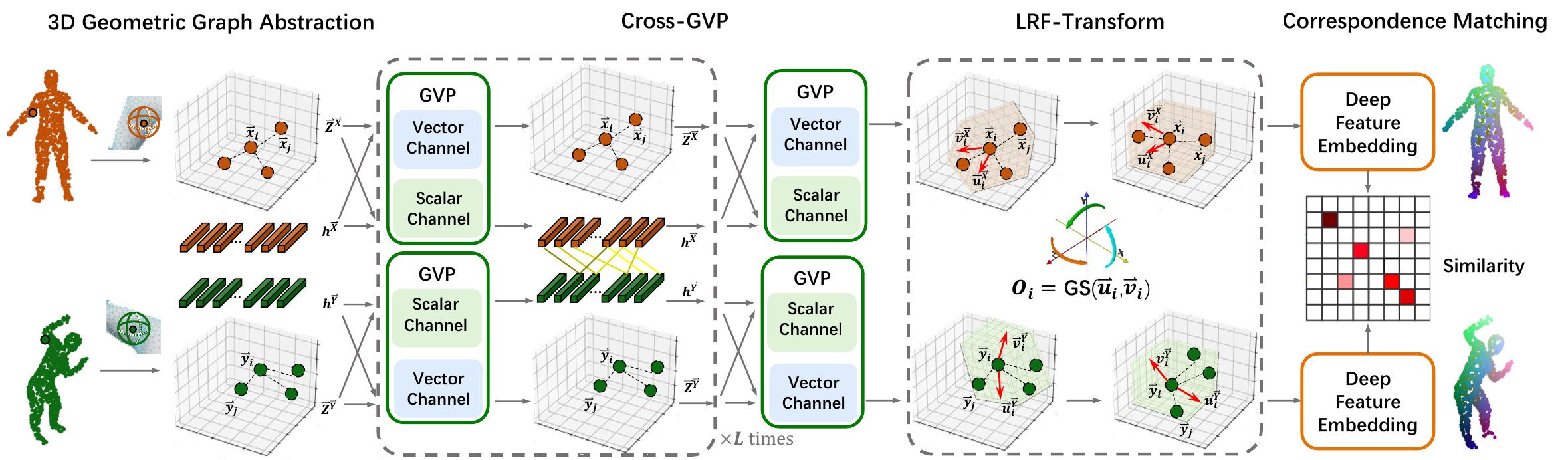}
\vspace{-1.5em}
\caption{\textbf{Illustrative flowchart of \textsc{EquiShape}}. 3D Geometric Graph is a graph equipped with the attribute $\Vec{\mZ}$, signifying 3D geometric vector features (\eg, global positions, relative directions) that are steerable via \(\text{SE}(3)\) transformations. 
The notation $\vh$ denotes non-steerable scalar features, \eg, distances, embedding features. 
\(\vec{\vu}_i, \vec{\vv}_i\) are the pair-wise independent SE(3)-equivariant LRF vectors output by Cross-GVP.
The function $\operatorname{GS}(\cdot,\cdot)$ constructs per-point pairwise LRF $\mO_i$ by applying the Gram-Schmidt orthogonalization process to these LRF vectors.
}
\label{fig:main}
\vspace{-1.5em}
\end{figure*}

Next, the processed inputs are passed to the Geometric Vector Perceptron (GVP) graph convolution layer~\cite{jing2020gvp,satorras2021en}, denoted as GVP-G:
\begin{equation}
\label{eq:gvp-g}
\aligned
(\vh^{(l)}_i, \vec{\mZ}^{(l)}_i) = \operatorname{GVP-G}_{l}(\vh^{(l-1)}_i, \vec{\mZ}^{(l-1)}_i, \Ecal_{\Vec{\mX}},\vh_{ij}, \vec{\mZ}_{ij}), 
i \in \Vcal_{\Vec{\mX}}, (i,j)\in\gE_{\Vec{\mX}}.
\endaligned
\end{equation}

Similarly, for target shapes $\Vec{\mY}$, the aforementioned definitions apply.


To enable dynamic adaptation in response to counterpart shape,  we design a cross-talk mechanism to broaden the scope of the contextual clues for LRF learning. Specifically, by integrating a cross-attention module to aggregate invariant scale features of both shapes to correlate paired inputs effectively, while keeping the independent equivariance property of the two shapes, node representations are updated as follows:
\begin{equation}
\label{eq:qkv}
\aligned
    \boldsymbol{\mu}_{i} &=  \sum_{j\in\Vcal_{\Vec{\mY}}}a_{j\to i} \Wb \hb_j^{(l)}, \forall i \in\Vcal_{\Vec{\mX}},  
\endaligned
\end{equation}
\begin{align}
\label{eq:bid-mu}
    \vh^{(l)}_i&=[\vh^{(l)}_i, \boldsymbol{\mu}_{i}], \forall i \in\Vcal_{\Vec{\mX}}\cup\Vcal_{\Vec{\mY}},
\end{align}
where \(\Wb\) is a learnable matrix, and \(a_{j \to i}\) is the cross-attention coefficient determined by trainable shallow neural networks \(\psi^q\) and \(\psi^k\):
\begin{align}
\label{eq:cross-att}
    a_{j\to i} &= \frac{\text{exp}(\langle \psi^q(\hb_i^{(l)}),\psi^k(\hb_j^{(l)})\rangle)}{\sum_{j'}\text{exp}(\langle \psi^q( \hb_i^{(l)}), \psi^k(\hb_{j'}^{(l)})\rangle)},
\end{align}
where $\langle\cdot,\cdot\rangle$ is the inner product of two vectors.
Note that all parameters of $\operatorname{GVP-G}_{l}, \Wb, \psi^q, \psi^k$  are shared for the same layers between shapes. The operations are stacked over $L$ layers in total.

With multiple layers of message fusion on the graphs of the pair of shapes, we can extract the pair-wise independent SE(3)-equivariant LRF vectors \((\vec{\vu}_i, \vec{\vv}_i)\) from the output of the GVP layer:
\begin{equation}
\label{eq:gvp}
\aligned
(\vec{\vu}_i, \vec{\vv}_i)\in\R^{3\times 2} = \operatorname{GVP}(\vh^{(L)}_i, \vec{\mZ}^{(L)}_i), i \in \Vcal_{\Vec{\mX}}\cup\Vcal_{\Vec{\mY}}.
\endaligned
\end{equation}

The output of the Cross-GVP is thus defined as:
\begin{equation}
\aligned
    \left((\vec{\vu},\vec{\vv})\in\R^{3\times 2\times 2 \times n} \right) = 
    \operatorname{Cross-GVP}(\Vec{\mX}, \Vec{\mY}).
\endaligned
\end{equation}


\paragraph{Parameterized Local Reference Frame Transform (LRF-Transform)}

We then employ the Gram-Schmidt orthogonalization process~\cite{luo2022equivariant} to normalize and orthogonalize the two rotation vectors for each point:
\begin{equation}
\label{eq:ortho}
\aligned
    \vec{\ve}_{i1} & = \frac{\vec{\vu}_{i}}{\| \vec{\vu}_{i} \|}, \\
    \vec{\ve}_{i2} & = \frac{\vec{\vv}_{i} - \langle\vec{\vv}_{i},\vec{\ve}_{i1}\rangle \vec{\ve}_{i1}}{\| \vec{\vv}_{i} - \langle\vec{\vv}_{i},\vec{\ve}_{i1}\rangle \vec{\ve}_{i1} \|}, \\
    \vec{\ve}_{i3} & = \vec{\ve}_{i1} \times \vec{\ve}_{i2}. 
\endaligned
\end{equation}
Here $\langle\cdot,\cdot\rangle$ is the inner product of two vectors and $\times$ is the cross product.  We refer to the aforementioned procedure as GS, following which we proceed to construct a pair-wise independent equivariant LRF $\mO$ for each point:
\begin{equation}
\label{eq:ort_matrix}
    \mO_i = \left[ \vec{\ve}_{i1}, \vec{\ve}_{i2}, \vec{\ve}_{i3} \right] = \operatorname{GS}(\vec{\vu}_{i}, \vec{\vv}_{i}), \quad i \in \Vcal_{\Vec{\mX}}\cup\Vcal_{\Vec{\mY}}.
\end{equation}
Although we utilize GVP~\cite{jing2020gvp}, an E(3)-equivariant network, the cross product operator in GS can transform E(3)-equivariant vectors into SE(3)-equivariant matrices (see proof in \Cref{sec:proof}).




At this stage, we establish the LRF for each point \( i \), with \(\Vec{\vx}_i\) as the origin and $\vec{\ve}_{i1}$ as the orientation of the $x$-axis, that decouples the relative position for each point from the global pose of the entire point cloud. Then we apply parameterized transformation to achieve local  \(\text{SE}(3)\)-invariant input features:
\begin{equation}
\label{eq:prl}
\aligned
    \{\vh_i\} = \{\operatorname{Agg} \sigma (\mO^\intercal_i(\Vec{\vx}_j - \Vec{\vx}_i))\} = \operatorname{Transform}(\mO, \Vec{\mX}), (i,j)\in\gE_{\Vec{\mX}},\\
\endaligned
\end{equation}
where $\operatorname{Agg}$ is a permutation-invariant aggregation operator (common choices include max, sum, and mean), $\sigma$ is an MLP layer.

In the next step, we feed the learned \(\text{SE}(3)\)-invariant features into the classical graph neural networks for shape correspondence (commonly DGCNN~\cite{zeng2021corrnet3d,lang2021dpc}):
\begin{equation}
\label{eq:dgcnn}
\aligned
    F_{\Vec{\mX}}\in\R^{c\times n} = \operatorname{DGCNN}(\{\vh_i\}), i \in \gV_{\Vec{\mX}}\\
\endaligned
\end{equation}

The definitions for target shapes $\mY$ are similar. These above processes result in the final architecture of \textsc{EquiShape}, with the full flowchart visualized in \Cref{fig:main}. $\varphi_{\theta}$ can be concretely instantiated as \textsc{EquiShape}. The output of \textsc{EquiShape} is then denoted as:
\begin{equation}
\label{eq:EquiShape}
\aligned
(F_{\Vec{\mX}}, F_{\Vec{\mY}}) \in\R^{c\times 2 \times n} 
= \operatorname{\textsc{EquiShape}}(\Vec{\mX}, \Vec{\mY}).
\endaligned
\end{equation}

\begin{proposition}
\label{prop:EquiShape}
    \textsc{EquiShape} satisfy the constraints in \Cref{eq:SE(3)-invariance}.
\end{proposition}
\begin{proof}
See \Cref{sec:proof}.
\end{proof}

\paragraph{Training Objective}
To learn point representation suitable for non-rigid point cloud shape correspondence without ground-truth supervision, we apply unsupervised training objectives following the previous work~\cite{lang2021dpc}.
%
%

We obtain the point clouds $\Vec{\mY}_c, \Vec{\mX}_s, \Vec{\mX}_c, \Vec{\mY}_s$ by cross-construction and self-construction.
Then, we constrain the training with the construction loss as follows:
\begin{equation}
    \label{construction_loss}
    \mathcal{L}_{cons}=\lambda_{cc}(\operatorname{CD}(\Vec{\mY}, \Vec{\mY}_c) + \operatorname{CD}(\Vec{\mX}, \Vec{\mX}_c)) 
    + \lambda_{sc}(\operatorname{CD}(\Vec{\mY}, \Vec{\mY}_s) + \operatorname{CD}(\Vec{\mX}, \Vec{\mX}_s)),
\end{equation}
where $\lambda_{cc}, \lambda_{sc}$ are hyperparameters and
$\operatorname{CD}$ means the Chamfer Distance.

Then, we constrain the training with the mapping loss as follows:
\begin{equation}
    \label{regularization}
    \mathcal{L}_{map}=\lambda_{m}(\mathcal{L}_{m}(\Vec{\mX}, \Vec{\mY}_c) + \mathcal{L}_{m}(\Vec{\mY}, \Vec{\mX}_c)),
\end{equation}
where $\lambda_{m}$ is a hyperparameter, and $\mathcal{L}_{m}$ denotes the cross-neighborhoods regularization term. To sum up, the overall unsupervised training objective in the dataset $\gD$ is:
\begin{align}
    \label{overall_loss}
    \mathcal{L}_{total} &=  \mathcal{L}_{cons} + \mathcal{L}_{map},\\
    \theta^* &= \argmin_{\theta}\E_{\{\Vec{\mX},\Vec{\mY}\}\sim\gD} \mathcal{L}_{total}(\Vec{\mX},\Vec{\mY}) .
\end{align}

More details on training objectives
are shown in  \Cref{sec:details}.

\subsection{Local Reference Frames Refinement Optimization}
\label{subsec:transformation_refinement}

\paragraph{Overview} 
For non-rigid tasks, the exponential complexity stemming from inter-point pose transformations renders the task ill-posed. This complexity ensures that learning-based methods struggle to cover all distributions during the training phase, inevitably leading to significant OOD challenges at the inference stage. 
To overcome this hurdle,  recent works \cite{lang2023scoop,hong2022neural,hatem2023point} directly refine coordinates, akin to per-point warping, to enhance generalizability, yet they optimize the model's final output solely based on loss, without the constraints of model knowledge, which, in our experiments~\Cref{subsec:refine}, are observed to be prone to poor improvement.
Moreover, since LRFs dictate whether correspondence local regions can match within an aligned coordinate space, we underscore the vital role of LRFs in enhancing the geometric and semantic generalizability of point features. 
To this end, we address this by directly refining the LRF vectors \((\vec{\vu}_i, \vec{\vv}_i)\), derived from Cross-GVP, adapted to specific geometric contexts under the constraints of the model during the inference stage. The following paragraphs detail the LRFs refinement process.


\begin{figure}[!t]
\centering
      \includegraphics[width=\textwidth]{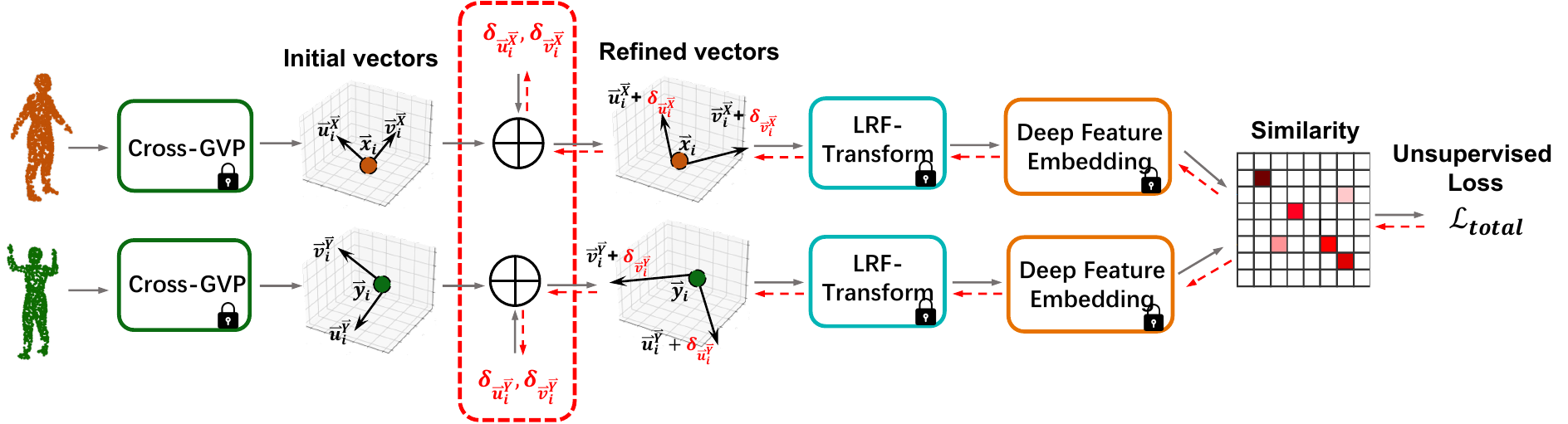}
      \vspace{-1.5em}
      \caption{\textbf{Illustrative flowchart of LRF-Refine}. Dashed lines represent back-propagating gradients, while locks indicate that parameters are frozen.}  
      
      \vspace{-1.em}
      \label{fig:optimization}
\end{figure}


\paragraph{Optimization Approach}
Specifically, we employ gradient descent to \textit{directly} optimize the residual LRF refinement vectors, aiming to minimize the loss defined in \Cref{overall_loss}. This process adjusts the LRF vectors in response to the specific contexts of OOD input shape, facilitating local matching within the refined LRFs. By focusing on refining LRFs rather than correspondence coordinates, our method ensures the preservation of learned matching knowledge while enhancing feature generalizability to novel local distortions.

\paragraph{Optimization Formulation} The refinement is formulated as an optimization problem:

\begin{equation} 
\label{eq:transformation_refinement}
\begin{aligned}
(\vec{\vu}_i, \vec{\vv}_i)  &= (\vec{\vu}_i + \delta_{\vec{\vu}_i}, \vec{\vv}_i + \delta_{\vec{\vv}_i}), \\
\delta^* &= \argmin_{\delta} \mathcal{L}_{\text{total}}(\Vec{\mX},\Vec{\mY}).
\end{aligned}
\end{equation}

Here, \(\delta = \{\delta_{\vec{\vu}_i}, \delta_{\vec{\vv}_i}\}_{i \in \Vcal_{\Vec{\mX}}\cup\Vcal_{\Vec{\mY}}}\) represents the set of residual LRF refined vectors for each pair of source and target shapes in the dataset \(\{\Vec{\mX},\Vec{\mY}\} \in \gD\). 
The objective \(\mathcal{L}_{\text{total}}\), as detailed in \Cref{overall_loss}, is utilized to optimize these vectors. This process adapts LRFs to specific, unseen pairs of shapes under the constraints of the frozen models, as illustrated in \Cref{fig:refine}.

%% file: sec/5-experiment.tex
\section{Experiments}
\label{sec:exp}

\subsection{Experimental Setup}
\paragraph{Baselines}  
We contrast our method with the recent state-of-the-art point-based matching methods 
CorrNet3D~\cite{zeng2021corrnet3d}, DPC~\cite{lang2021dpc} and SE-ORNet~\cite{deng2023se}. 
These methods uniformly utilize point clouds as input and do not require ground truth for supervision. Note that the results of all the examined baselines are reproduced using their publicly available official source code and checkpoints.
Please refer to  \Cref{sec:details} for more details about baselines.

\paragraph{Datasets}
To showcase the effectiveness and broad applicability of our method, we conduct experiments across a diverse range of datasets, as shown on ~\Cref{tbl:dataset}.
Among these, SURREAL~\cite{groueix20183d} , SHREC’19~\cite{melzi2019shrec}, and CAPE~\cite{cape1}  are human datasets, while SMAL~\cite{zuffi20173d} and TOSCA~\cite{bronstein2008numerical} are animal datasets. All of them are categorized as non-rigidity datasets. Notably, SHREC’19 and CAPE exhibit orientation discrepancies between source-target shapes, with CAPE additionally being a real-world dataset.
The details of all datasets are introduced in the  \Cref{sec:details}.

\paragraph{Metrics}
The evaluation metrics include the correspondence accuracy and the average correspondence error.

The correspondence accuracy can be formulated for pairs of source and target shapes $\{\Vec{\mX}^t,\Vec{\mY}^t\} \in \gD^t$  as:
\begin{equation}
	\label{correspondence accuracy}
\aligned
  &\text{acc}(\epsilon)=
  &\E_{\{\Vec{\mX}^t,\Vec{\mY}^t\}\sim\gD^t}\frac{1}{|\Vcal_{\Vec{\mX}^t}|} \sum_{ i \in \Vcal_{\Vec{\mX}^t}} \mathbb{I}\left(\left\|f\left(\Vec{\vx}_{i}\right)-\Vec{\vy}_i^{g t}\right\|_{2}<\epsilon d\right),
\endaligned
\end{equation}
where $\mathbb{I} (\cdot)$ is the indicator function, $d$ is the maximal Euclidean distance between points in $\Vec{\mY}^t$, and $\epsilon \in [0, 1]$ is an error tolerance, $\Vec{\vy}_i^{g t} \in \Vec{\mY}^t$ is the ground-truth matching point to $\Vec{\vx}_{i}$. 

Based on the Euclidean-based measure, the average correspondence error is defined  as follows:
\begin{equation}
	\label{correspondence error}
	\text{err}=\E_{\{\Vec{\mX}^t,\Vec{\mY}^t\}\sim\gD^t}\frac{1}{|\Vcal_{\Vec{\mX}^t}|} \sum_{ i \in \Vcal_{\Vec{\mX}^t}} \left\|f\left(\Vec{\vx}_{i}\right)-\Vec{\vy}_i^{g t}\right\|_{2},
\end{equation}
where the unit is centimeter (cm).

\begin{table}[!t]
\centering
\setlength{\tabcolsep}{2.8pt}
\caption{\textbf{Dataset statistics.} We report the number of training, validation, and test samples used in the quantitative evaluations.}
\label{tbl:dataset}
\footnotesize
\vspace{-.7em}
\resizebox{0.7\textwidth}{!}{%
\begin{tabular}{l|cccccc}
\toprule
Dataset           & \# Train /val & \# Test & Category & Rigidity & Align & Real \\ \midrule
SURREAL~\cite{varol2017learning}          & 1600/400           & /           &    human    & \ding{55}  & \ding{51}       & \ding{55}          \\
SHREC'19~\cite{melzi2019shrec} & 1513/379          & 430          & human       & \ding{55}        &  \ding{55} &  \ding{55}    \\
SMAL~\cite{zuffi20173d}       & 8000/2000           & /           & animal       & \ding{55}   & \ding{51}     &  \ding{55}         \\
TOSCA~\cite{bronstein2008numerical}           & /             & 286           &  animal      &   \ding{55}  & \ding{51}      &   \ding{55}          \\
CAPE~\cite{cape1}           & /             & 209           &  human      &   \ding{55}   & \ding{55}     &   \ding{51}     \\\bottomrule
\end{tabular}
}
\vspace{-.8em}
\end{table}

\paragraph{Implementations}
Our non-rigid shape graphs are constructed using $k=27$ neighboring nodes. In the Cross-GVP module, we utilize $L=3$ GVP-G layers to generate \(\text{SE}(3)\)-equivariant LRF vectors. Each layer features have $d=64$ hidden feature channels. 
We implement the Cross-GVP module by adapting the open-source variant of GVP-based graph convolutional networks (GVP-GCN) \cite{jing2020gvp, luo2022equivariant} combined with cross-attention \cite{vaswani2017attention,ganea2021independent}. At last, we adapt the open-source DGCNN~\cite{wang2019dynamic} implementation for our shape correspondence module, where the output channels of $\operatorname{EdgeConv}$ are $\{64, 64, 128, 256, 512\}$, which is same as the $\operatorname{EdgeConv}$ of CorrNet3D Encoder~\cite{zeng2021corrnet3d}.
Please refer to  \Cref{sec:details} for detailed model parameter statistics and more implementation details.

\begin{table*}[!t]
\centering
    \footnotesize
    \setlength\tabcolsep{7pt}
    \caption{ \textbf{Comparisons with State-of-the-Art methods.} Here, ``acc'' represents the correspondence accuracy at an error tolerance of 0.01, whereas ``err'' indicates the average correspondence error measured in centimeters (cm). A higher ``acc'' coupled with a lower ``err'' metric signifies superior performance. ``A/B'' indicates training on the A dataset and testing on the B dataset.
    The results for other methods are reproduced using their officially released source code and checkpoints.
    }
    \label{table:sota}
    \vspace{-0.7em}
    \resizebox{\textwidth}{!}{
    \begin{tabular}{c|c|cc|cc|cc|cc}
      \toprule
      \multirow{2}{*}{Methods}  & \multirow{2}{*}{Params} & \multicolumn{2}{c|}{SHREC'19~/~SHREC'19} & \multicolumn{2}{c|}{SURREAL~/~SHREC'19} & \multicolumn{2}{c|}{SMAL~/~TOSCA} & \multicolumn{2}{c}{SURREAL~/~CAPE}\\
      \cline{3-4}\cline{5-6}\cline{7-8}\cline{9-10}
      & &  acc $\uparrow$ & err $\downarrow$ & acc $\uparrow$ & err $\downarrow$ & acc $\uparrow$ & err $\downarrow$ & acc $\uparrow$ & err $\downarrow$ \\
      \midrule
      CorrNet3D\cite{zeng2021corrnet3d}          & 5.4M & /  & / & 7.4\% & 18.3 &/ & / & 28.3\%& 9.3\\
      DPC\cite{lang2021dpc}                     & 1.5M& 14.4\%  & 14.4 & 17.4\% & 15.0 &33.8\% &7.4 & 50.5\% & 5.6 \\
      SE-ORNet\cite{deng2023se}                     & 2.5M & 16.9\%  & 14.0 & 21.8\% & 11.8 &38.1\% & 5.2 & 43.2\% & 8.0\\
        \textbf{Ours}                              &  \textbf{762K} &\textbf{22.1\%}  & \textbf{13.7} & \textbf{30.3\%} &  \textbf{10.1} & \textbf{57.7\%}& \textbf{3.0} & \textbf{68.4\%} & \textbf{2.8} \\
      \bottomrule
    \end{tabular}
    }
    \vspace{-.5em}
\end{table*}

\subsection{Comparisons with State-of-the-Arts}
Our primary quantitative results are illustrated in \Cref{table:sota}. 
Particularly noteworthy is the performance in cross-dataset evaluations. 
Our method outperforms the baselines by a large margin on both human and animal benchmarks, affirming its robustness and broad applicability.
In addition, we assessed the generality of our approach on the real-world CAPE dataset, as shown in \Cref{table:sota}  (right). Our method achieves a remarkable 68.4\% accuracy with a mere 2.8 cm error—outshining SE-ORNet's 43.2\% accuracy and 8.0 cm error.  As depicted in \Cref{fig:acc}, it consistently exceeds baselines across varying error tolerances, establishing its reliability even under strict precision constraints. 
Observing \Cref{fig:vis}, our method exhibits superior visualization results compared to all baselines, particularly on the SHREC'19, which features misalignment in orientation.

\begin{figure}[!t]
  \begin{center}
      \includegraphics[width=\linewidth]{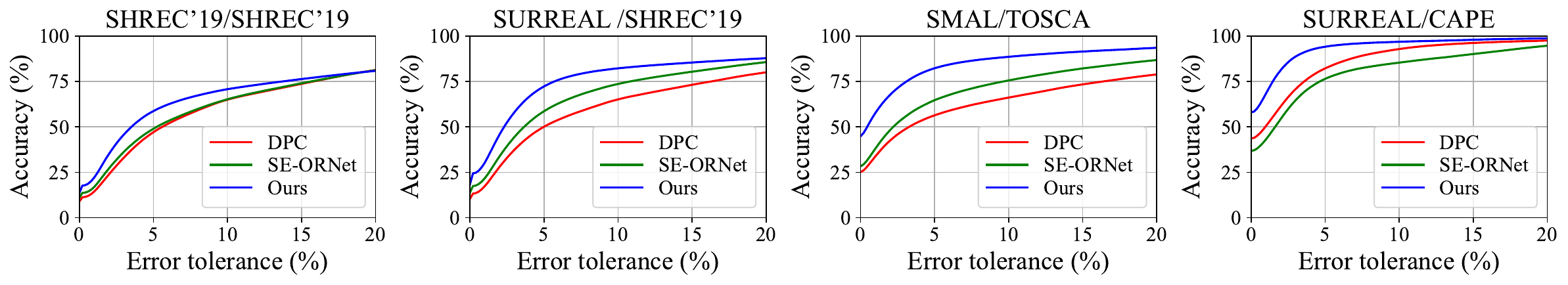}
      \vspace{-1.5em}
      \caption{\textbf{Correspondence accuracy at various error tolerances.} 
      }
      \vspace{-1.5em}
      \label{fig:acc}
  \end{center}
\end{figure}

\begin{figure}[!t]
\begin{minipage}[b]{0.5\textwidth}
  \begin{center}
      \includegraphics[width=0.95\textwidth]{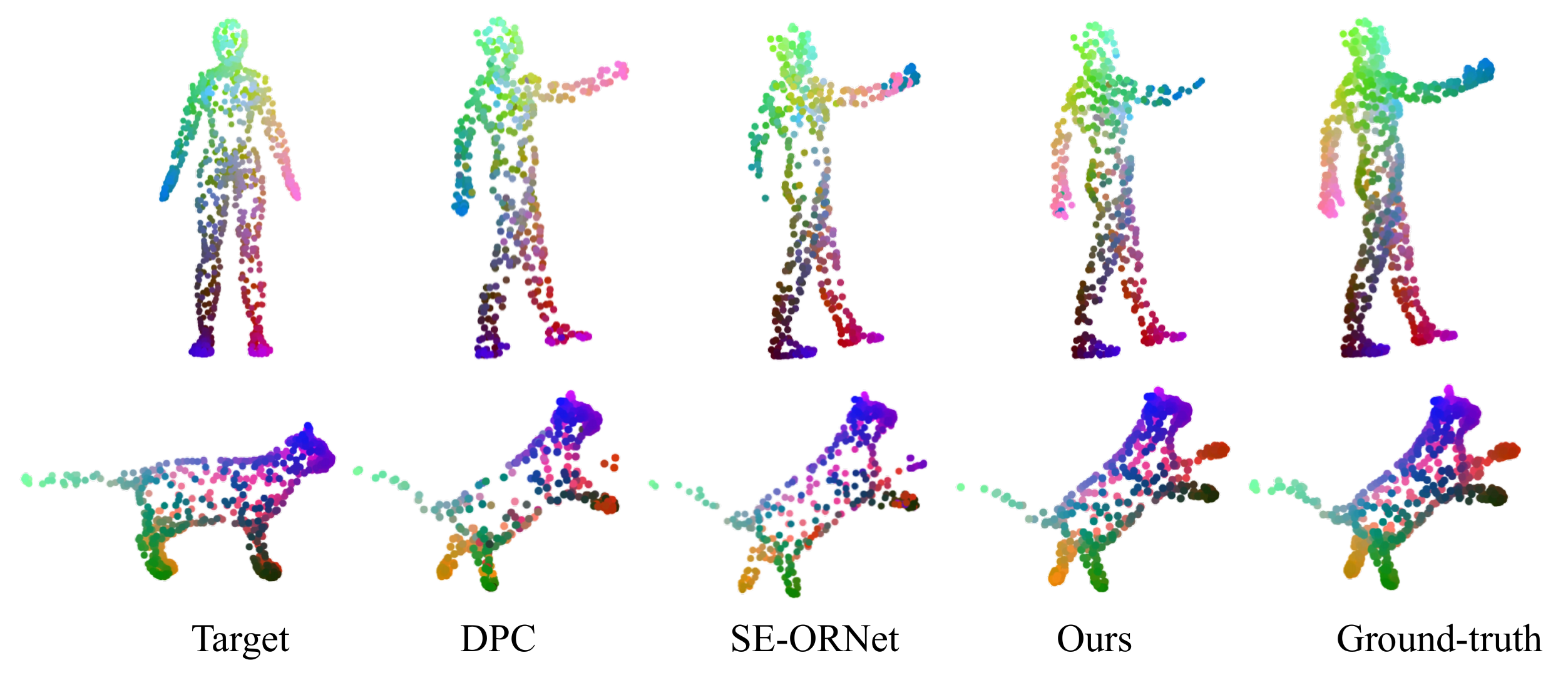}
      \vspace{-.5em}
      \caption{\textbf{Visualizations of the correspondence results} from SHREC'19 (top) and TOSCA (bottom) test set.
      }
      \vspace{-1.5em}
      \label{fig:vis}
  \end{center}
  \end{minipage}
  \hspace{0.5em}
  \begin{minipage}[b]{0.46\textwidth}
  \begin{center}
      \includegraphics[width=0.9\linewidth]{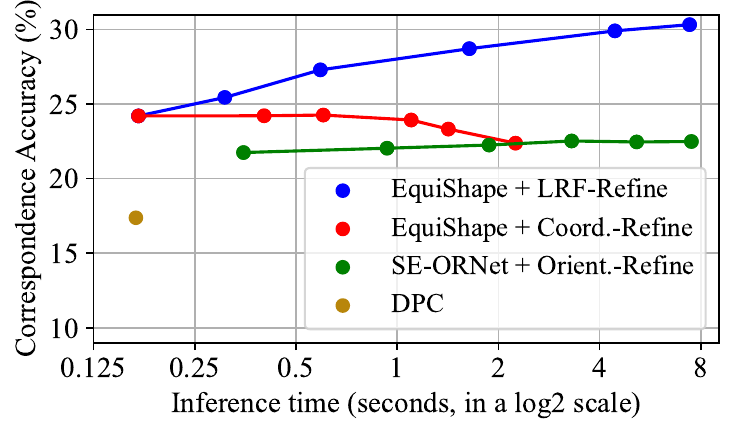}
      \vspace{-.5em}
      \caption{\textbf{Refinement and time efficiency analysis.} 
      }
      \vspace{-1.5em}
      \label{fig:refine}
  \end{center}
  \end{minipage}
\end{figure}

\subsection{Ablation Study}
\label{subsec:ablation}
We further conduct extensive ablation studies to elucidate the design choices. 
1) Rows (a) and (b) utilize invariant descriptors output by equivariant networks, namely Vector Neuron Networks (VNN)~\cite{deng2021vector} and GVP-GCN~\cite{jing2020gvp}, respectively.
Compared to VNN, commonly used in point cloud tasks, GVP-GCN boosts performance by leveraging a graph neural network to capture contextual correlations.
2) \textit{LRF}. 
Row (e) utilizes GVP-GCN, as described in EquivariantMP~\cite{luo2022equivariant}, to end-to-end learn the LRF. In contrast, Row (c) manually constructs the LRF following the methodology outlined in SHOT~\cite{tombari2010unique}, demonstrating the superiority of learnable LRFs over handcrafted ones. Furthermore, the comparison of Row (c) to Row (b) illustrates that applying LRFs to local regions explicitly decomposes the structure, thereby enhancing performance.
3) \textit{Refine}. 
The application of our proposed refinement strategy in Rows (d), (f), and (h) yields superior outcomes, demonstrating the refinement strategy's generally applicable for LRF-based methods and its significant impact on addressing OOD generalization.
 4) \textit{Cross-talk}. 
Row (g) further demonstrates that performance in shape matching significantly improves when leveraging cross-talk to seek global contextual clues compared to Row (f). 
We further explore the impact of hyperparameter $k$ of $k$NN in Cross-GVP. 
As demonstrated in \cref{table:knn}, a small \(k\) leads to insufficient geometric context and less distinctive features, while a large \(k\) disrupts local spatial consistency, deteriorating performance.

\begin{table}[!t]
\centering
    \footnotesize
    \setlength\tabcolsep{4pt}
    \caption{\textbf{Alation Study.} We consistently employ the DGCNN as the feature extractor and identical training objectives. 
    ``EquivariantNet'' refers to the use of an equivariant network prior to DGCNN. ``LRF'' indicates the construction of LRFs before DGCNN. ``Refine'' denotes the optimization of LRFs during the inference stage. ``Cross-talk'' signifies the use of cross-attention within GVP-GCN.
    All ablative methods are trained on SURREAL and evaluated on SHREC'19. }
    \vspace{-0.7em}
    \resizebox{0.9\columnwidth}{!}{
    \begin{tabular}{l|cccc|c}
      \toprule
     Method & EquivariantNet & LRF &   Refine & Cross-talk &  acc $\uparrow$ \\
      \midrule
      (a)~VNN~\cite{deng2021vector}&\ding{51}&\ding{55}&\ding{55}&\ding{55}&9.5\%\\
      (b)~GVP-GCN~\cite{jing2020gvp}&\ding{51}&\ding{55}&\ding{55}&\ding{55}&11.4\%\\
      (c)~SHOT~\cite{tombari2010unique}&\ding{55}&\ding{51}&\ding{55}&\ding{55}&15.1\%\\
      (d)~SHOT~\cite{tombari2010unique} + LRF-Refine~&\ding{55}&\ding{51}&\ding{51}&\ding{55}&20.3\%\\
      
       (e)~EquivariantMP~\cite{luo2022equivariant}&\ding{51}&\ding{51}&\ding{55}&\ding{55}& 19.9\%\\
         (f)~EquivariantMP~\cite{luo2022equivariant} + LRF-Refine~&\ding{51}&\ding{51}&\ding{51}&\ding{55}& 25.4\%\\
      (g)~\textsc{EquiShape}&\ding{51}&\ding{51}&\ding{55}&\ding{51}& 24.2\%\\
      (h)~\textbf{\textsc{EquiShape} + LRF-Refine}&\ding{51}&\ding{51}&\ding{51}&\ding{51}&\textbf{30.3\%}\\
      \bottomrule
    \end{tabular}
   }
    \label{table:ablation}
    \vspace{-.5em}

\end{table}

\begin{table}[!t]
\centering
    \setlength\tabcolsep{2.5pt}
    \caption{\textbf{Hyperparameter Analysis of $k$NN}: Models trained on SURREAL and assessed for accuracy at an error tolerance of 0.01, on SHREC'19 test set. 
    }

    \label{table:knn}
    \vspace{-0.8em}
    \resizebox{0.4\columnwidth}{!}{
    \begin{tabular}{c|cccc}
      \toprule
     \ $k$  & 10  &  27 &  50 & 100   \\ 
      \midrule
       acc (\% ) $\uparrow$ & 14.0     & \textbf{30.3}     &  27.2   &19.5   \\
      \bottomrule
    \end{tabular}
    }
    \vspace{-1.5em}

\end{table}


\subsection{Refinement and Time Analysis}
\label{subsec:refine}
We analyze the influence of the refinement process on enhancing generalizability. We compare three distinct refinement strategies: ``\textsc{EquiShape} + Coord.-Refine'', which directly refines correspondence coordinates following the approach in SCOOP~\cite{lang2023scoop} within our framework; ``SE-ORNet + Orient.-Refine'', which refines the estimated relative orientation angles within the SE-ORNet framework; and our proposed ``\textsc{EquiShape} + LRF-Refine''. 
As illustrated in \Cref{fig:refine}, \textsc{EquiShape}, augmented with the proposed LRF-Refine process, not only achieves higher correspondence accuracy compared to ``SE-ORNet + Orient.-Refine'' but also exhibits a notable reduction in inference time. 
Conversely, ``\textsc{EquiShape} + Coord.-Refine'' shows limited improvement, potentially due to the absence of model knowledge guidance.
The blue curve in the graph portrays our method’s flexibility, offering a continuum where users can choose the optimal compromise between error margins and computational speed.

%% file: sec/6-conclusion.tex
\section{Conclusion}
\label{sec:Conclusion}

In conclusion, our work introduces \textsc{EquiShape} and LRF-Refine to address the profound challenges in unsupervised non-rigid point cloud shape correspondence, stemming from exponential complexity and out-of-distribution geometric contexts. By innovatively leveraging pair-wise independent SE(3)-equivariant LRFs through Cross-GVP and refining these LRFs to adapt to specific contexts, we significantly enhance both the spatial and semantic generalizability of point features. Our approach not only sets a new benchmark by outperforming existing methods by a considerable margin on standard datasets but also pioneers the integration of equivariant networks in non-rigid shape matching. This paves the way for future research in achieving effective matching across the vastly diverse scenarios encountered in 3D vision tasks.

%% file: sec/X_suppl.tex
\clearpage
\appendix
\setcounter{page}{1}
\setcounter{proposition}{0}
\setcounter{section}{0}

\section{Proofs}
\label{sec:proof}
In this section, we theoretically prove that our proposed \textsc{EquiShape} satisfy the constraints in \Cref{eq:SE(3)-invariance} as desired.

\begin{theorem}
\label{proof:Cross-GVP}
The learned vectors, denoted as 
\begin{align*}
\left((\vec{\vu}^{\Vec{\mX}},\vec{\vv}^{\Vec{\mX}}), (\vec{\vu}^{\Vec{\mY}},\vec{\vv}^{\Vec{\mY}}) \right) = \operatorname{Cross-GVP}(\Vec{\mX}, \Vec{\mY}),
\end{align*}
are \(\text{SE}(3)\)-equivariant, satisfying 
any transformation $g_1,  g_2\in\text{SE}(3)$, 
\begin{align*}
\left((g_1 \cdot \vec{\vu}^{\Vec{\mX}},g_1 \cdot \vec{\vv}^{\Vec{\mX}}), (g_2 \cdot\vec{\vu}^{\Vec{\mY}},g_2 \cdot \vec{\vv}^{\Vec{\mY}}) \right) =
\operatorname{Cross-GVP}(g_1 \cdot \Vec{\mX}, g_2 \cdot\Vec{\mY}).
\end{align*}
\end{theorem}
\begin{proof}

Consider the equations updating the geometric vectors $\vec{\mZ}^{(l)}_i, \vec{\vu}_i, \vec{\vv}_i$ in \Cref{eq:edge,eq:init,eq:gvp-g,eq:gvp}. These updates are \(\text{SE}(3)\)-equivariant with respect to the inputs from the same graph as node $i$ but  \(\text{SE}(3)\)-invariant with respect to the inputs from the other graph. This is because the cross-attention mechanism in \Cref{eq:qkv,eq:bid-mu,eq:cross-att} utilizes invariant non-steerable features from the other graph. Consequently, the output vectors $\left((\vec{\vu}^{\Vec{\mX}},\vec{\vv}^{\Vec{\mX}}), (\vec{\vu}^{\Vec{\mY}},\vec{\vv}^{\Vec{\mY}}) \right)$ obtained from $\operatorname{Cross-GVP}(\Vec{\mX}, \Vec{\mY})$ satisfy any transformation $g_1,  g_2\in\text{SE}(3)$, 
\begin{align*}
\left((g_1 \cdot \vec{\vu}^{\Vec{\mX}},g_1 \cdot \vec{\vv}^{\Vec{\mX}}), (g_2 \cdot\vec{\vu}^{\Vec{\mY}},g_2 \cdot \vec{\vv}^{\Vec{\mY}}) \right) =
\operatorname{Cross-GVP}(g_1 \cdot \Vec{\mX}, g_2 \cdot\Vec{\mY}).
\end{align*}
\end{proof}

\begin{theorem}
\label{proof:eqo}
The learned pair-wise LRF for each point, denoted as $ \mO_i = \operatorname{GS}(\vec{\vu}_{i}, \vec{\vv}_{i})$, are \(\text{SE}(3)\)-equivariant, satisfying any transformation $g \in\text{SE}(3)$, $g \cdot  \mO_i = \operatorname{GS}(g \cdot \vec{\vu}_{i}, g \cdot  \vec{\vv}_{i})$. 
\end{theorem}
\begin{proof}
To prove that the LRF $\mO_i = \operatorname{GS}(\vec{\vu}_{i}, \vec{\vv}_{i})$ are \(\text{SE}(3)\)-equivariant, we need to show that under any transformation $g \in \text{SE}(3)$, the transformation of $\mO_i$ through $g$ is equivariant to the LRF obtained from the transformed vectors $g \cdot \vec{\vu}_{i}$ and $g \cdot \vec{\vv}_{i}$.

Let $g$ be a transformation in $\text{SE}(3)$, which includes a rotation $\mO$ and a translation $\vec{\vt}$. 
Specifically, the transformation is applied as follows:
\begin{align*}
    g \cdot \mO_i &= \mO \mO_i, \\
    g \cdot \vu_i &= \mO \vu_i, \\
    g \cdot \vv_i &= \mO \vv_i.
\end{align*}

Let $\mO_i' = \operatorname{GS}(g \cdot  \vec{\vu}_{i}, g \cdot  \vec{\vv}_{i})$. By the properties of the Gram-Schmidt process, we have:
\begin{align*}
\vec{\ve}_{i1}' &= \frac{\mO \vec{\vu}_{i}}{\|\mO \vec{\vu}_{i}\|}, \\
\vec{\ve}_{i2}' &= \frac{\mO \vec{\vv}_{i} - \langle \mO \vec{\vv}_{i}, \vec{\ve}_{i1}' \rangle \vec{\ve}_{i1}'}{\|\mO \vec{\vv}_{i} - \langle \mO \vec{\vv}_{i}, \vec{\ve}_{i1}' \rangle \vec{\ve}_{i1}'\|}, \\
\vec{\ve}_{i3}' &= \vec{\ve}_{i1}' \times \vec{\ve}_{i2}'.
\end{align*}

Since $\mO^\top\mO=\mI$ and $\det(\mO) = 1$
, it preserves inner products and norms. And, the cross product obeys the following identity under matrix transformations:
$(M \Vec{\va}) \times(M \Vec{\vb})=(\operatorname{det} M)\left(M^{-1}\right)^{\mathrm{T}}(\Vec{\va} \times \Vec{\vb})$.
Therefore, $\|\mO \vec{\vu}_{i}\| = \|\vec{\vu}_{i}\|$ and $\langle \mO \vec{\vv}_{i}, \mO \vec{\vu}_{i} \rangle = \langle \vec{\vv}_{i}, \vec{\vu}_{i} \rangle$. This implies:
\begin{align*}
\vec{\ve}_{i1}' &= \frac{\mO \vec{\vu}_{i}}{\|\vec{\vu}_{i}\|} = \mO \vec{\ve}_{i1}, \\
\vec{\ve}_{i2}' &= \frac{\mO \vec{\vv}_{i} - \langle \mO\vec{\vv}_{i}, \mO\vec{\ve}_{i1} \rangle \mO \vec{\ve}_{i1}}{\|\mO \vec{\vv}_{i} - \langle \mO\vec{\vv}_{i}, \mO\vec{\ve}_{i1} \rangle \mO \vec{\ve}_{i1}\|}
=\frac{\mO \vec{\vv}_{i} - \langle \vec{\vv}_{i}, \vec{\ve}_{i1} \rangle \mO \vec{\ve}_{i1}}{\|\mO \vec{\vv}_{i} - \langle \vec{\vv}_{i}, \vec{\ve}_{i1} \rangle \mO \vec{\ve}_{i1}\|} = \mO \vec{\ve}_{i2}, \\
\vec{\ve}_{i3}' &= \mO \vec{\ve}_{i1} \times \mO \vec{\ve}_{i2} = \det(\mO)\mO (\vec{\ve}_{i1} \times \vec{\ve}_{i2}) = \mO \vec{\ve}_{i3}.
\end{align*}

Therefore, with $\mO_i' = \mO \mO_i = g \cdot \mO_i$, we confirm the \(\text{SE}(3)\)-equivariance of $\mO_i$. 

Besides, for reflection transformations, characterized by $\mO^\top\mO=\mI$ and $\det(\mO) = -1$, the equivariance is disrupted due to the properties of the cross product. Hence, although we utilize GVP, an E(3)-equivariant network, the Gram-Schmidt process can transform E(3)-equivariant vectors into SE(3)-equivariant matrices.

\end{proof}

\begin{theorem}
\label{proof:Transform}
The learned features, denoted as $ \{\vh^{\Vec{\mX}}_i\} = \operatorname{Transform}(\mO^{\Vec{\mX}}, \Vec{\mX}),  i\in\gV_{\Vec{\mX}}$, are \(\text{SE}(3)\)-invariant, satisfying any transformation $g, g\in\text{SE}(3)$, $ \{\vh^{\Vec{\mX}}_i\} = \operatorname{Transform}(g \cdot \mO^{\Vec{\mX}}, g \cdot \Vec{\mX})$. 
Similarly,  $\operatorname{Transform}(\mO^{\Vec{\mY}}, \Vec{\mY})$ are \(\text{SE}(3)\)-invariant.
\end{theorem}
\begin{proof}
To prove that the learned features $\{\vh^{\Vec{\mX}}_i\} = \operatorname{Transform}(\mO^{\Vec{\mX}}, \Vec{\mX}),  i\in\gV_{\Vec{\mX}}$, are \(\text{SE}(3)\)-invariant, we need to demonstrate that under any transformation $g \in \text{SE}(3)$, the transformation of $\{\vh^{\Vec{\mX}}_i\}$ through $g$ is equivariant to the features obtained from the transformed LRF and points, \ie, $g \cdot \mO^{\Vec{\mX}}$ and $g \cdot \Vec{\mX}$.

Let $g$ be a transformation in $\text{SE}(3)$, which includes a rotation $\mO$ and a translation $\vec{\vt}$. 
Specifically, for a point $i \in \gV_{\Vec{\mX}}$ and its corresponding LRF $\mO_i$, the transformation is applied as follows:
\begin{align*}
    g \cdot \mO_i &= \mO \mO_i, \\
    g \cdot \vx_i &= \mO \vx_i + \vec{\vt}.
\end{align*}

Applying the Transform operation on the inputs, we get:
\begin{align*}
    &\operatorname{Agg} \sigma \left(\mO_i^\intercal \mO^\intercal (\mO \vx_j + \vec{\vt} - (\mO \vx_i + \vec{\vt}))\right) \\
    &= \operatorname{Agg} \sigma \left(\mO_i^\intercal \mO^\intercal \mO (\vx_j - \vx_i)\right)\\
    &= \operatorname{Agg} \sigma \left(\mO_i^\intercal (\vx_j - \vx_i)\right).
\end{align*}

Thus, $ \{\vh^{\Vec{\mX}}_i\} = \operatorname{Transform}(g \cdot \mO^{\Vec{\mX}}, g \cdot \Vec{\mX})$, proving that $\operatorname{Transform}(\mO^{\Vec{\mX}}, \Vec{\mX})$ is \(\text{SE}(3)\)-invariant. The same reasoning applies to $\operatorname{Transform}(\mO^{\Vec{\mY}}, \Vec{\mY})$, confirming its \(\text{SE}(3)\)-invariance.
\end{proof}

\begin{proposition}
\label{proof:EquiShape}
    \textsc{EquiShape} satisfy the constraints in \Cref{eq:SE(3)-invariance}.
\end{proposition}
\begin{proof}
To demonstrate that \textsc{EquiShape} satisfies the constraints in \Cref{eq:SE(3)-invariance}, we need to show that the similarity matrix $S(\cdot, \cdot)$ computed from the learned features $F_{\Vec{\mX}}, F_{\Vec{\mY}}$ is invariant under any transformations $g_1, g_2 \in \text{SE}(3)$ applied to the input shapes ${\Vec{\mX}}, {\Vec{\mY}}$.

Recall that \textsc{EquiShape} comprises several components, each contributing to the overall invariance of the model. These components include the Cross-GVP, which provides equivariant vectors, the LRF-Transform, which ensures \(\text{SE}(3)\)-invariance, and a classical graph neural network (such as DGCNN), which further processes these invariant features.

Given transformations $g_1, g_2 \in \text{SE}(3)$, we apply them to the input shapes to get $g_1\cdot{\Vec{\mX}}, g_2\cdot{\Vec{\mY}}$. Due to the \(\text{SE}(3)\)-invariance of the Transform and the subsequent processing by DGCNN, the high-dimensional point representations $F_{\Vec{\mX}}, F_{\Vec{\mY}}$ remain invariant under these transformations. That is,
\begin{align*}
    F_{g_1\cdot{\Vec{\mX}}}, F_{g_2\cdot{\Vec{\mY}}} &= \operatorname{\textsc{EquiShape}}\left(g_1\cdot{\Vec{\mX}}, g_2\cdot{\Vec{\mY}} \right)\\
    &= \operatorname{\textsc{EquiShape}}\left({\Vec{\mX}}, {\Vec{\mY}} \right) = F_{\Vec{\mX}}, F_{\Vec{\mY}}.
\end{align*}

Consequently, the similarity matrix $S(F_{\Vec{\mX}}, F_{\Vec{\mY}})$ remains the same irrespective of the transformations applied to the input shapes:
\begin{align*}
    S\left(F_{g_1\cdot{\Vec{\mX}}}, F_{g_2\cdot{\Vec{\mY}}}\right) = S\left(F_{\Vec{\mX}}, F_{\Vec{\mY}}\right),
\end{align*}
which confirms the invariance of the similarity matrix $S(\cdot, \cdot)$ under any transformations in \(\text{SE}(3)\).

Therefore, \textsc{EquiShape} satisfies the constraints in \Cref{eq:SE(3)-invariance}, confirming that the learned correspondence, specifically the similarity matrix $S(\cdot, \cdot)$, is independent of the initial positions and orientations of the input shapes.
\end{proof}

\section{More Experimental Details}
\label{sec:details}

\noindent\textbf{Training Objective.}
To effectively learn point representation for non-rigid point cloud shape correspondence without relying on ground-truth supervision, we adopt self- and cross-reconstruction losses based on latent neighborhoods, following the approach outlined in previous work~\cite{lang2021dpc}. The cross-construction component aims to foster unique point matches between shape pairs, whereas the self-construction component serves as a regularizer, promoting smoothness in the correspondence map. We facilitate the construction of the target shape $\Vec{\mY}_c$ using the feature similarity $S(F_{\Vec{ \mX}}, F_{\Vec{\mY}})$, defined as:
\begin{equation}
    \label{apx:construction}
    \Vec{\vy}_{\Vec{\vx}_{i}}=\sum_{j \in \mathcal{N}_{\Vec{\mY}}\left(\Vec{\vx}_{i}\right)} \frac{e^{s_{i j}}}{\sum_{l \in \mathcal{N}_{\Vec{\mY}}(\Vec{\vx}_{i})} {e^{s_{i l}}}} \Vec{\vy}_{j},
\end{equation}
where $\mathcal{N}_{\Vec{\mY}}(\Vec{\vx}_{i})$ denotes the latent k-nearest neighbors of $\Vec{\vx}_{i}$ in the target $\Vec{\mY}$, and $s_{i j}$ represents the elements of the similarity matrix $S(F_{\Vec{\mX}}, F_{\Vec{\mY}})$.  When the source and target shapes are identical, we refer to the construction operation as self-construction.

Through cross-construction and self-construction, we generate point clouds $\Vec{\mY}_c, \Vec{\mX}_s, \Vec{\mX}_c, \Vec{\mY}_s$. The training is then constrained by the construction loss:
\begin{equation}
\begin{aligned}
    \label{apx:construction_loss}
    \mathcal{L}_{cons}=&\lambda_{cc}(\operatorname{CD}(\Vec{\mY}, \Vec{\mY}_c) + \operatorname{CD}(\Vec{\mX}, \Vec{\mX}_c)) \\
    &+ \lambda_{sc}(\operatorname{CD}(\Vec{\mY}, \Vec{\mY}_s) + \operatorname{CD}(\Vec{\mX}, \Vec{\mX}_s)),
\end{aligned}
\end{equation}
where $\lambda_{cc}$ and $\lambda_{sc}$ are hyperparameters, and $\operatorname{CD}$ denotes the Chamfer Distance.

We also introduce a regularization term to ensure that neighboring points in the source correspond to proximate points in the target:
\begin{equation}
    \label{apx:regularization}
    \mathcal{L}_{m}(\Vec{\mX}, \Vec{\mY}_c)=\sum_{i} \sum_{l \in \mathcal{N}_{\Vec{\mX}}(\Vec{\vx}_{i}) } e^{-\left\|\Vec{\vx}_{i}-\Vec{\vx}_{l}\right\|_{2}^{2} / \alpha} \left\|\Vec{\vy}_{\Vec{\vx}_{i}}-\Vec{\vy}_{\Vec{\vx}_{l}}\right\|_{2}^{2},
\end{equation}
where $\alpha$ is a hyperparameter. This loss is similarly defined for the mapping from the target to the source, namely, $\mathcal{L}_{m}(\Vec{\mY}, \Vec{\mX}_c)$.

Consequently, the mapping loss constrains the training as follows:
\begin{equation}
    \label{apx:mapping_loss}
    \mathcal{L}_{map}=\lambda_{m}(\mathcal{L}_{m}(\Vec{\mX}, \Vec{\mY}_c) + \mathcal{L}_{m}(\Vec{\mY}, \Vec{\mX}_c)),
\end{equation}
where $\lambda_{m}$ is a hyperparameter.

The total unsupervised training objective in the train dataset $\gD^s$ is thus composed of:
\begin{align}
    \label{apx:overall_loss}
    \mathcal{L}_{total} &=  \mathcal{L}_{cons} + \mathcal{L}_{map},\\
    \theta^* &= \argmin_{\theta}\E_{\{\Vec{\mX}^s,\Vec{\mY}^s\}\sim\gD^s} \mathcal{L}_{total}(\Vec{\mX}^s,\Vec{\mY}^s) .
\end{align}

\noindent\textbf{Baselines.}
Our approach is compared with recent state-of-the-art point-based matching methods: CorrNet3D~\cite{zeng2021corrnet3d}, DPC~\cite{lang2021dpc}, and SE-ORNet~\cite{deng2023se}. We ensure fairness in our evaluation by using the officially released source code and checkpoints for each baseline. Detailed descriptions of these baselines are as follows:

\begin{itemize}
    \setlength{\leftskip}{1em}
    \item \textbf{CorrNet3D}~\cite{zeng2021corrnet3d}: This method aims to learn a pairwise deformation field in an unsupervised manner, leveraging shape reconstruction and cycle consistency. Despite its innovative approach, CorrNet3D's deformation network demands extensive training data, thus limiting its generalizability.
    \item \textbf{DPC}~\cite{lang2021dpc}: DPC forgoes a deformation network, instead relying on feature extraction for point cloud similarity calculations to reconstruct shapes. It introduces multiple reconstruction losses to refine point cloud representations, focusing on smoother transitions.
    \item \textbf{SE-ORNet}~\cite{deng2023se}: This method enhances shape correspondence by estimating point cloud orientations and employing a domain-adaptive discriminator for alignment. Additionally, SE-ORNet incorporates a self-ensembling framework to obtain a more robust feature representation.
\end{itemize}

\noindent\textbf{Datasets.}
Our experiments span a diverse set of datasets (\Cref{tbl:sup_dataset}), validating the effectiveness of our method. The datasets and their configurations are as follows:

\begin{itemize}
    \setlength{\leftskip}{1em}
    \item \textbf{SURREAL}~\cite{varol2017learning}: A synthetic dataset generated from the SMPL human model~\cite{loper2023smpl}, comprising 230k shapes. For consistency with prior work~\cite{lang2021dpc,deng2023se}, we utilize a subset of 2k human shapes for training.
    \item \textbf{SHREC'19}~\cite{melzi2019shrec}: Known for its challenging variation in sampling density and mesh connectivity, this dataset includes 44 human shapes. We form 1892 training pairs from these shapes and evaluate using the official set of 430 annotated pairs.
    \item \textbf{SMAL}~\cite{zuffi20173d}: Featuring parametric models of five different animals, we create 2k shapes per category, amounting to 10k shapes. Training pairs are randomly selected across animal categories.
    \item \textbf{TOSCA}~\cite{bronstein2008numerical}: Generated from three base meshes (human, dog, and horse), we select 41 animal shapes to form 286 test samples within the same category.
    \item \textbf{CAPE}~\cite{cape1}: Consisting of real clothed human scans. We use 209 raw-scanned point cloud pairs, preprocessed as per ~\cite{huang2022multiway}, for evaluation.
    \item \textbf{DeepDeform}~\cite{bozic2020deepdeform}: Introduced in the supplementary material, this real-world dataset includes RGB-D frames of non-rigidly moving objects, characterized by noise and partiality. We use 50 animal and 35 human pairs from the test set preprocessed by ~\cite{huang2022multiway} for evaluation.
\end{itemize}

All datasets are standardized to a uniform point count of $n = 1024$ per shape, aligning with CorrNet3D~\cite{zeng2021corrnet3d} for fair comparison. 

\begin{table}[ht]
\centering
\setlength{\tabcolsep}{2.8pt}
\caption{\textbf{Dataset statistics.} We report the number of training, validation and test samples used in the quantitative evaluations.}
\label{tbl:sup_dataset}
\footnotesize
\vspace{-1.em}
\resizebox{0.8\textwidth}{!}{
\begin{tabular}{l|cccccccc}
\toprule
Dataset           & \# Train /val & \# Test & Category & Rigidity  & Align & Real & Partial \\ \midrule
SURREAL~\cite{varol2017learning}          & 1600/400           & /           &    human    & \ding{55}    & \ding{51}       & \ding{55}     & \ding{55}        \\
SHREC'19~\cite{melzi2019shrec} & 1513/379          & 430          & human         & \ding{55}      &  \ding{55} &  \ding{55}    & \ding{55}  \\
SMAL~\cite{zuffi20173d}       & 8000/2000           & /           & animal      & \ding{55}      & \ding{51}     &  \ding{55}       & \ding{55}    \\
TOSCA~\cite{bronstein2008numerical}           & /             & 286           &  animal     & \ding{55}   & \ding{51}      &   \ding{55}      & \ding{55}      \\
CAPE~\cite{cape1}           & /             & 209           &  human    & \ding{55}      & \ding{55}     &   \ding{51}    & \ding{55}   \\
DeepDeform-H~\cite{bozic2020deepdeform}           & /             & 35          &  human    & \ding{55}   & \ding{55}     &   \ding{51}    & \ding{51}   \\
DeepDeform-A~\cite{bozic2020deepdeform}           & /             & 50          &  animal   & \ding{55}      & \ding{55}     &   \ding{51}   & \ding{51}    \\
\bottomrule
\end{tabular}
}
\vspace{-1.5em}
\end{table}

\vspace{1.0em}

\noindent\textbf{Implementations.}
Our non-rigid shape graphs are constructed using $k=27$ neighboring nodes. In the CrossGVP module, we utilize $L=3$ GVP-G layers to generate \(\text{SE}(3)\)-equivariant LRF vectors. Each layer features have $d=64$ hidden feature channels. The LRF-Transform module incorporates a \textit{max} aggregation operator for pooling deformation features from node neighbors.

The basic model configuration, Ours-B, aligns with the DGCNN configuration of CorrNet3D~\cite{zeng2021corrnet3d}, consisting of a 5-layer EdgeConv with channels $\{64, 64, 128, 256, 512\}$. We also introduce Ours-L, with enhanced channel specifications $\{96, 192, 384, 768, 512\}$, in parallel with DPC~\cite{lang2021dpc} and SE-ORNet~\cite{deng2023se}. Each EdgeConv layer integrates batch normalization and LeakyReLU activation (negative slope $0.2$). Hyperparameters for the losses are set as $\lambda_{cc} = 1, \lambda_{sc} = 10, \lambda_m = 1$.

Implemented in PyTorch ~\cite{paszke2019pytorch}, our model uses the Adam optimizer ~\cite{kingma2014adam} with a batch size of 8 on an NVIDIA A100 GPU, spanning 300 epochs. The initial learning rate is 3e-4, reduced tenfold at epochs 6 and 9. Detailed information on the model's trainable parameters and computational complexity (FLOPs) is provided in \Cref{table:sota_large}.

For refinement, the component $\delta$, as defined in \Cref{eq:transformation_refinement}, is initialized with zero vectors and optimized using Adam ~\cite{kingma2014adam} with a learning rate of 1e-8 over 100 steps.


\section{Additional Results}

\subsection{Additional Evaluations on Partial DeepDeform Benchmark}
To further assess the adaptability of our model to real-world sensor data, we conduct performance evaluations on the DeepDeform dataset~\cite{bozic2020deepdeform}. This dataset features an array of real-world partial RGB-D scans of dynamic subjects, including humans, animals, cloth, etc. We directly test our trained models on human and animal categories. As \Cref{table:deepdeform} details, our method significantly outperforms baseline approaches. Notably, in the animal category, our method exceeds SE-ORNet, the second-best model, by over 40 percentage points in accuracy. This pronounced distinction is visually evident in \Cref{fig:dd_vis}. Such a marked improvement, especially amidst the diverse animal poses in the DeepDeform dataset, underscores our model's efficacy and robustness. Our SE(3)-equivariant approach is highly effective in handling complex and previously unseen poses.

\begin{table}[ht]
  \begin{center}
    \footnotesize
    \setlength\tabcolsep{8pt}
    \caption{\textbf{Extended Evaluations on Partial DeepDeform Benchmark.} The acronym ``DD'' refers to DeepDeform datasets, with ``H'' denoting human subjects and ``A'' representing animal categories. 
  ``A/B'' indicates training on the A dataset and testing on the B dataset. The results for other methods are reproduced using their officially released source code and checkpoints.}
    \label{table:deepdeform}
    \vspace{-0.8em}
    \begin{tabular}{c|cc|cc}
      \toprule
      \multirow{2}*{Method}  & \multicolumn{2}{c|}{SURREAL/DD-H} & \multicolumn{2}{c}{SMAL/DD-A} \\
      \cline{2-3}\cline{4-5}
      & acc $\uparrow$ & err $\downarrow$ & acc $\uparrow$ & err $\downarrow$ \\
      \midrule
      DPC\cite{lang2021dpc}                      & 57.0\%  & 5.3 & 9.6\% & 16.5 \\
      SE-ORNet\cite{deng2023se}          & 52.9\%  & \textbf{1.5} & 9.8\% & 4.5  \\
      \textbf{Ours}                              & \textbf{72.3\%}  & \textbf{1.5} & \textbf{51.0\%} & \textbf{2.8} \\
      \bottomrule
    \end{tabular}
  \end{center}
\end{table}


\begin{figure}[ht]
  \begin{center}
      \includegraphics[width=0.7\textwidth]{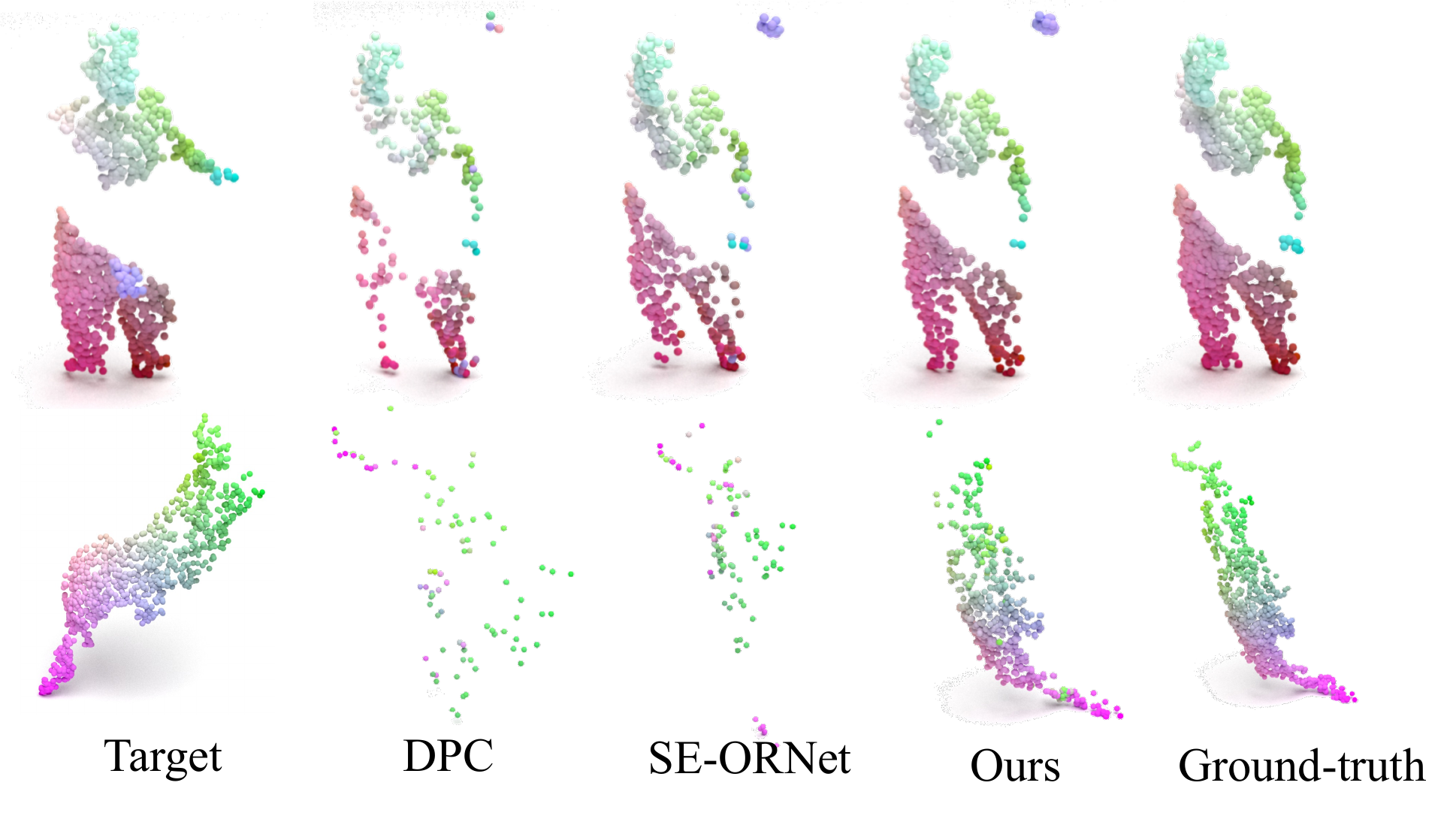}
      \caption{\textbf{Additional visualizations of the correspondence results} from DeepDeform test set.
      }
      \vspace{-1.5em}
      \label{fig:dd_vis}
  \end{center}
\end{figure}

\subsection{Additional Evaluations with Different Orientations}
\label{subsec:se3-invariance-analysis}
In the domain of 3D shape analysis, the ability of a model to generalize across various orientations is paramount. SE-ORNet~\cite{deng2023se} improves its orientation estimation for rotations around a fixed axis through data augmentation techniques. While it experiences a slight performance drop under yaw rotations, SE-ORNet nevertheless surpasses DPC~\cite{lang2021dpc} in performance. However, its accuracy under SE(3) transformations remains relatively low. Conversely, our \textsc{EquiShape} demonstrates considerable robustness to arbitrary orientation transformations, achieving 24.2\% accuracy even without the refinement process, as evidenced by \Cref{table:Invariance}. This accuracy increases to 30.3\% with the refinement process. As depicted in \Cref{fig:vis_se3},  we apply SE(3) transformations to target shapes from the CAPE test set. The visualizations distinctly showcase the disparity in performance between our approach and baseline methods, especially in handling orientation alterations. While baseline models often struggle with misaligning left and right hand orientations, our method adeptly preserves semantic alignment in these critical areas.

\begin{table}[ht]
\centering
    \setlength\tabcolsep{6pt}
    \caption{\textbf{SE(3)-Invariance Analysis on Test Dataset}: Models trained on SURREAL and assessed for correspondence accuracy, at an error tolerance of 0.01, on SHREC'19 test set. ``SE(3)'' refers to the application of random SE(3) transformations to the source shapes.  Conversely, ``yaw'' denotes the imposition of rotations exclusively around the $y$-axis on these source shapes.
    }

    \label{table:Invariance}
    \vspace{-0.8em}
    \resizebox{\columnwidth}{!}{
    \begin{tabular}{c|ccc}
      \toprule
      ~ &SHREC'19  &  SHREC'19(yaw) &  SHREC'19(SE(3))   \\ 
      \midrule
       DPC~\cite{lang2021dpc}& 17.7\%     & 9.1\%     &  3.0\%      \\
   SE-ORNet~\cite{deng2023se} &  21.5\%     &   15.1\%    & 2.6\% \\
   \textsc{EquiShape}   &   24.2\%   &   24.0\%   & 24.1\%  \\
   \textbf{\textsc{EquiShape} + LRF-Refine} &  \textbf{30.3\%}   &  \textbf{30.6 \%}   & \textbf{30.3\%}   \\
      \bottomrule
    \end{tabular}
    }

\end{table}

\begin{figure}[ht]
  \begin{center}
      \includegraphics[width=0.65\textwidth]{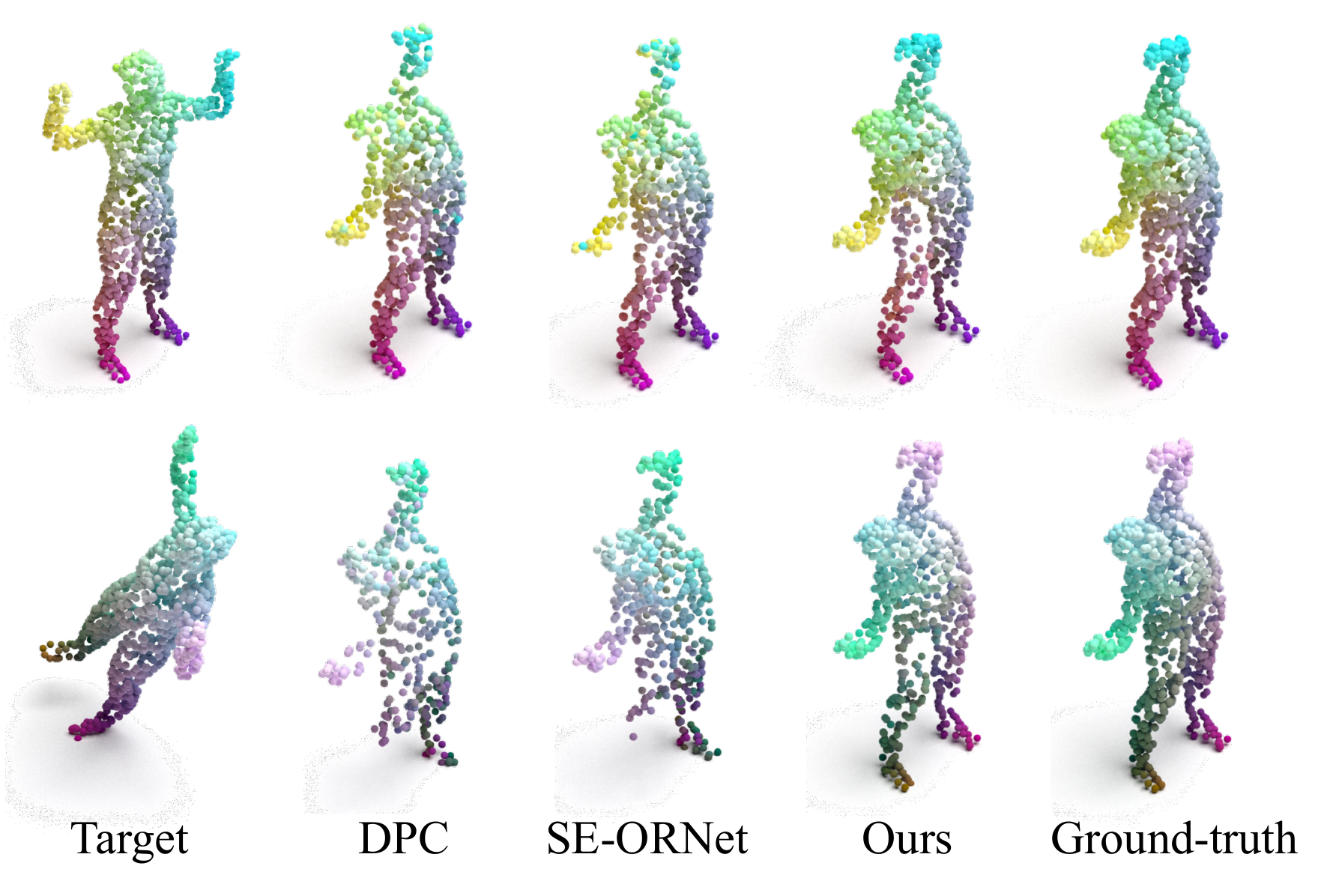}
      \caption{\textbf{Additional visualizations of the correspondence results with SE(3)-invariance analysis} from CAPE test set.
      }
      \vspace{-1.5em}
      \label{fig:vis_se3}
  \end{center}
\end{figure}

\subsection{Novelty compared to EquivariantMP} 
EquivariantMP~\cite{luo2022equivariant} is a general framework to achieve equivariance for point cloud analysis. It can be configured in two ways for shape matching tasks. The first configuration relies on dependent SE(3)-equivariance, wherein the point clouds of the two shapes to be matched, $\Vec{\mX}$ and $\Vec{\mY}$, are merged into a single geometric graph $\Gcal_{\Vec{\mX} \cup \Vec{\mY}} = (\Vcal_{\Vec{\mX} \cup \Vec{\mY}}, \Ecal_{\Vec{\mX} \cup \Vec{\mY}})$. In this graph, each node $i \in \Vcal_{\Vec{\mX} \cup \Vec{\mY}}$ represents a point with 3D location coordinates $\vec{\vp}_i \in {\Vec{\mX} \cup \Vec{\mY}}$. The graph topology $\Ecal_{\Vec{\mX} \cup \Vec{\mY}}$ is constructed using a $k$-nearest-neighbor ($k$NN) approach based on the Euclidean distance within the combined 3D point clouds ${\Vec{\mX}} \cup {\Vec{\mY}}$. The desired invariance is formulated as:
\begin{equation}
    S\left(F_{g\cdot{\Vec{\mX}}}, F_{g\cdot{\Vec{\mY}}}\right) \equiv  S\left(F_{\Vec{\mX}}, F_{\Vec{\mY}}\right),
    \label{eq:dependent SE(3)-invariance}
\end{equation}
where $g \in \text{SE}(3)$ and ${\Vec{\mX}}, {\Vec{\mY}}\in\R^{3\times n}$. However, modeling such dependent SE(3)-equivariance presents considerable challenges compared to single objects, primarily due to the infinite shape variations resulting from relative orientation transformations with exponentially growing complexities.

In contrast, our work employs another independent SE(3)-equivariant configuration, aiming for the predicted LRF to remain invariant regardless of the initial positions and orientations of the shapes. As described in \Cref{sec:preli}, we construct a pair of graphs $\Gcal_{\Vec{\mX}}$ and $\Gcal_{\Vec{\mY}}$, with the desired invariance being specified as in \Cref{eq:SE(3)-invariance}. Theoretically, this independent setting significantly reduces computational complexity. Empirical results from our experiments, as illustrated in \Cref{table:independent}, demonstrate the superior performance of the independent configuration over the dependent setting.
\begin{table}[ht]
\centering
    \setlength\tabcolsep{2.5pt}
    \caption{\textbf{Analysis of Independence}: The acronym ``D'' refers to dependent configuration, with ``I'' denoting independent configuration. Models trained on SURREAL and assessed for accuracy at an error tolerance of 0.01, on SHREC'19 test set. 
    }

    \label{table:independent}
    \vspace{-0.8em}
    \resizebox{0.65\columnwidth}{!}{
    \begin{tabular}{c|ccc}
      \toprule
     \ Method  & EquivariantMP(D)  & EquivariantMP(I) &  \textsc{EquiShape}    \\ 
      \midrule
       acc (\% ) $\uparrow$ & 6.1     & 19.9    &  \textbf{24.2}    \\
      \bottomrule
    \end{tabular}
    }
    \vspace{-1.5em}
\end{table}

Although the independent configuration for EquivariantMP effectively reduces computational complexity, it sacrifices some valuable associative information, notably the spatial relationships between shapes. To circumvent this limitation, we have ingeniously designed a cross-talk mechanism. This mechanism preserves the independent equivariance property whilst incorporating global geometric contexts---encompassing both the shape in question and its counterpart---thereby augmenting the distinctiveness of the predicted LRFs. Specifically, by integrating global information from absolute coordinates and employing message passing, our approach captures long-range semantic information, thereby enriching the model's semantic understanding beyond mere local structural details.

Furthermore, we leverage cross-attention to aggregate invariant scale features across shapes, synthesizing novel, globally-aware equivariant and invariant features. This methodology not only retains the independence of the equivariance property between shapes but also broadens the spectrum of contextual cues accessible for LRF learning. Through detailed ablation studies on the cross-talk mechanism, as evidenced in \Cref{table:cross-talk}, the efficacy of our design is substantiated.



\begin{table}[ht]
\centering
    \setlength\tabcolsep{2.5pt}
    \caption{\textbf{Analysis of Cross-talk}: The acronym ``I'' refers to independent configuration. Models trained on SURREAL and assessed for accuracy at an error tolerance of 0.01, on SHREC'19 test set.}
    \label{table:cross-talk}
    \vspace{-0.8em}
    \resizebox{0.8\columnwidth}{!}{
    \begin{tabular}{c|cccc}
      \toprule
     \multirow{2}{*}{Method}  &  \multirow{2}{*}{EquivariantMP(I)} & EquivariantMP(I)  & EquivariantMP(I)  & \multirow{2}{*}{\textsc{EquiShape}}    \\ 
      &  & + global coord. & + cross-atten. &     \\
      \midrule
       acc (\% ) $\uparrow$ & 19.9     & 22.5   & 20.6 &  \textbf{24.2}    \\
      \bottomrule
    \end{tabular}
    }
    \vspace{-1.5em}
\end{table}




Finally, in response to the inherent challenges posed by OOD geometric contexts, exacerbated by extensive shape variations, we have incorporated LRF-Refine, an optimization strategy applicable to LRF-based methods. This refinement process adjusts the LRF vectors to specific inputs,  under the guidance of model constraints, thereby substantially improving the geometric and semantic generalizability of point features, as evidenced in \Cref{table:refine}.

\begin{table}[ht]
\centering
    \setlength\tabcolsep{2.5pt}
    \caption{\textbf{Analysis of LRF-Refine}: The acronym ``I'' refers to independent configuration. Models trained on SURREAL and assessed for accuracy at an error tolerance of 0.01, on SHREC'19 test set.}
    \label{table:refine}
    \vspace{-0.8em}
    \resizebox{0.8\columnwidth}{!}{
    \begin{tabular}{c|cccc}
      \toprule
     \multirow{2}{*}{Method}  &  \multirow{2}{*}{EquivariantMP(I)} & EquivariantMP(I)  & \multirow{2}{*}{\textsc{EquiShape}}  & \textsc{EquiShape}    \\ 
      &  & + LRF-Refine. &  &  + LRF-Refine   \\
      \midrule
       acc (\% ) $\uparrow$ & 19.9     & 25.4   & 24.2 &  \textbf{30.3}    \\
      \bottomrule
    \end{tabular}
    }
    \vspace{-1.5em}
\end{table}

\subsection{Additional Results with Enhanced EdgeConv Channels}
To enhance the capacity of our model, we retrain it with an expanded DGCNN, adopting channel configurations of $\{64, 64, 128, 256, 512\}$, namely Ours-L, aligning with the specifications of DPC~\cite{lang2021dpc} and SE-ORNet~\cite{deng2023se}. The comprehensive outcomes across various datasets are presented in \Cref{table:sota_large} and visually depicted in \Cref{fig:vis_large}. Impressively, the incorporation of enhanced EdgeConv channels in the Ours-L model leads to notable performance gains. Specifically, on the challenging SHREC'19 benchmark, Ours-L attains an impressive 27.1\% accuracy, marking a significant improvement of 22.6\% compared to the Ours-B model.

\begin{table*}[ht]
\centering
    \footnotesize
    \setlength\tabcolsep{7pt}
    \caption{\textbf{Additional comparisons with State-of-the-Art methods.}  ``acc'' represents the correspondence accuracy with an error tolerance of 0.01, and ``err'' indicates the average correspondence error in centimeters (cm). Higher ``acc'' and lower ``err'' values denote superior performance. Results for comparative methods were obtained using their officially released source codes and checkpoints. The term ``Params''  denotes the number of trainable parameters. The basic model configuration, Ours-B, aligns with the DGCNN configuration of CorrNet3D~\cite{zeng2021corrnet3d}, consisting of a 5-layer EdgeConv with channels $\{64, 64, 128, 256, 512\}$. We also introduce Ours-L, with enhanced channel specifications $\{96, 192, 384, 768, 512\}$, paralleling DPC~\cite{lang2021dpc} and SE-ORNet~\cite{deng2023se}.}
    \label{table:sota_large}
    \resizebox{\textwidth}{!}{
    \begin{tabular}{c|c|c|cc|cc|cc|cc}
      \toprule
      \multirow{2}{*}{Methods}  & \multirow{2}{*}{Params} & \multirow{2}{*}{FLOPs} & \multicolumn{2}{c|}{SHREC'19~/~SHREC'19} & \multicolumn{2}{c|}{SURREAL~/~SHREC'19} & \multicolumn{2}{c|}{SMAL~/~TOSCA} & \multicolumn{2}{c}{SURREAL~/~CAPE}\\
      \cline{4-5}\cline{6-7}\cline{8-9}\cline{10-11}
      & & & acc $\uparrow$ & err $\downarrow$ & acc $\uparrow$ & err $\downarrow$ & acc $\uparrow$ & err $\downarrow$ & acc $\uparrow$ & err $\downarrow$ \\
      \midrule
      CorrNet3D\cite{zeng2021corrnet3d}          & 5.4M & 17.9G &/  & / & 7.4\% & 18.3 &/ & / & 28.3\%& 9.3\\
      DPC\cite{lang2021dpc}                     & 1.5M&  44.7G &14.4\%  & 14.4 & 17.4\% & 15.0 &33.8\% &7.4 & 50.5\% & 5.6 \\
      SE-ORNet\cite{deng2023se}                     & 2.5M & 187.2G & 16.9\%  & 14.0 & 21.8\% & 11.8 &38.1\% & 5.2 & 43.2\% & 8.0\\
        Ours-B                              &  762K & 30.1G &22.1\%  & 13.7 & 30.3\% &  10.1 & 57.7\%& 3.0 & 68.4\% & 2.8 \\
        \textbf{Ours-L}                             &  1.9M & 70.1G &\textbf{27.1\%}  & \textbf{10.9} & \textbf{31.8\%} &  \textbf{9.7} & \textbf{59.6\%}& \textbf{2.7} & \textbf{69.2\%} & \textbf{2.8} \\
      \bottomrule
    \end{tabular}
    }
\end{table*}

\begin{figure}[ht]
  \begin{center}
      \includegraphics[width=0.75\textwidth]{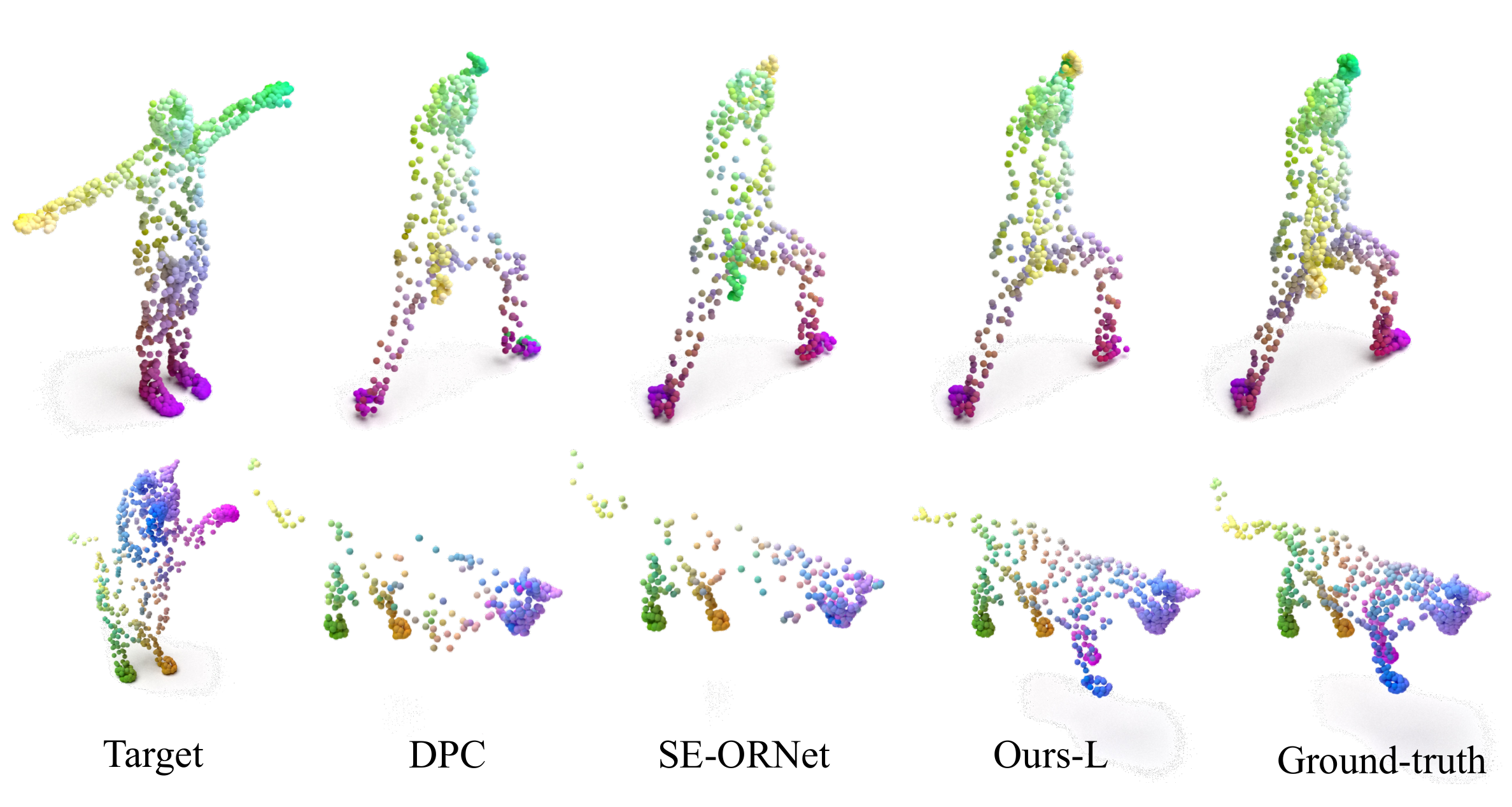}
      \caption{\textbf{Additional visualizations of the correspondence results} from SHREC'19 (top) and TOSCA (bottom) test set.
      }
      \vspace{-1.5em}
      \label{fig:vis_large}
  \end{center}
\end{figure}

\section{Limitations}
Our \textsc{EquiShape}, while effective, is subject to two potential limitations. Firstly,  our approach is fundamentally reliant on geometric graphs, which might present challenges in scenarios characterized by sudden geometric changes. In fact, this is a common issue for other methods based on the assumption of local rigidity.
Secondly, the reliance on unsupervised reconstruction losses for training might limit the method's effectiveness in low-overlap cases between shapes. Addressing these challenging scenarios remains an area for future exploration.